\documentclass{article} % For LaTeX2e
\usepackage{iclr2027_conference,times}

\usepackage{amsmath,amsfonts,bm}

\def\eqref#1{equation~\ref{#1}}
\def\1{\bm{1}}

\DeclareMathAlphabet{\mathsfit}{\encodingdefault}{\sfdefault}{m}{sl}
\SetMathAlphabet{\mathsfit}{bold}{\encodingdefault}{\sfdefault}{bx}{n}

\usepackage[utf8]{inputenc} % allow utf-8 input
\usepackage[T1]{fontenc}    % use 8-bit T1 fonts
\usepackage{hyperref}       % hyperlinks
\usepackage{url}            % simple URL typesetting
\usepackage{booktabs}       % professional-quality tables
\usepackage{amsfonts}       % blackboard math symbols
\usepackage{nicefrac}       % compact symbols for 1/2, etc.
\usepackage{microtype}      % microtypography
\usepackage{xcolor}         % colors
\usepackage{graphicx}
\usepackage{algorithm}
\usepackage{algpseudocode}
\usepackage{amsmath}
\usepackage{multirow}
\usepackage{amssymb}
\usepackage{pifont}
\usepackage{wrapfig}
\usepackage{caption}
\usepackage{subcaption}
\usepackage{amsthm}
\usepackage{thmtools}
\usepackage{thm-restate}
\usepackage{dsfont}

\usepackage{makecell}

\usepackage{enumitem}
\usepackage[most]{tcolorbox}
\newcommand*{\defeq}{\stackrel{\text{def}}{=}}

\usepackage[table]{xcolor}  % for \cellcolor and table color support
\usepackage{pifont}          % for \ding{51} and \ding{55}
\usepackage{colortbl}        % usually pulled in by xcolor[table]

\title{Soft Strategy Selection for \\Batch-Mode Active Learning}

\author{Rushil Gupta$^{1}$, Romain Lopez$^{1,2}$\\[4pt]
$^{1}$Department of Computer Science, Courant Institute School of Mathematics,\\
Computing, and Data Science, New York University\\
$^{2}$Department of Biology, New York University\\
\texttt{\{rushil.gupta, romain.lopez\}@nyu.edu}
}

\newcommand{\se}[1]{{\tiny$\pm$#1}}
\iclrfinalcopy % Uncomment for camera-ready version, but NOT for submission.
\begin{document}

\maketitle
\begin{abstract}
Real-world deployment of active learning (AL) typically forces practitioners to choose an acquisition strategy before any data is labeled. This is a daunting task: strategy performance varies widely across settings (e.g., datasets and surrogate model architectures) and cannot be assessed without deployment. Existing strategy selection methods explore one strategy from a portfolio at each round and identify the optimal one using bandit feedback or model retraining. Many acquisition rounds are therefore spent exploring strategies rather than collecting the most informative data. Such overhead is a major barrier to AL-driven design of high-throughput experiments, such as genetic perturbation screens and directed evolution, where AL runs consist of only a few rounds with large batch sizes. This regime permits a natural alternative: acquiring data using multiple AL strategies within a single batch. We refer to this as \textit{soft strategy selection} and introduce FractAL, a method specifically designed for this task. FractAL infers a per-strategy reward using influence-function-based data attribution, which requires no additional retraining, and then computes budget shares for each strategy in the portfolio using online mirror descent. We benchmark FractAL across seven active learning setups spanning classification, regression, and genetic perturbation effect prediction. The results highlight that strategy selection is a hard problem: every existing method performs worse than random sampling on at least one setup. FractAL, however, matches or outperforms every baseline, including random sampling, on all seven setups. Its allocations concentrate budget on the strongest strategies in the portfolio while pruning the weakest. FractAL is therefore a reliable choice for real-world deployments, where the optimal strategy is unknown in advance, an important step towards making active learning practical for high-throughput experiments and modern scientific discovery.
\end{abstract}

\section{Introduction}
Designing informative scientific experiments is a central challenge across the life and chemical sciences, arising in directed evolution of proteins~\citep{tran2026rapid}, antibody engineering~\citep{frey2025lab}, perturbation biology~\citep{rood2024toward}, and molecular design~\citep{smith2018less}. In these fields, researchers decide which genes, proteins, antibodies, or compounds to include in each experiment. Active Learning (AL) automates this process. At each experimental round $t \in [T]$, an acquisition strategy $\mathcal{A}$ selects a batch $\mathcal{B}^t$ of size $B$ from the input pool $\mathcal{D}_\text{pool}$. Then, outcomes are measured at selected inputs, and the surrogate model $f_\theta$ is retrained on the augmented training set $\mathcal{D}_\text{train}$~\citep{settles2009active, ren2021survey}. The goal of a strategy is to maximize model performance under experimental budget constraints. 

Many AL strategies have been proposed, targeting model uncertainty~\citep{houlsby2011bayesian}, batch diversity~\citep{sener2017active}, pool representativeness~\citep{yehuda2022active,hacohen2022active}, or combinations of these. However, the number of strategies keeps growing, and no single one performs well across domains, datasets, and model architectures~\citep{munjal2022towards, hacohen2022active, luth2023navigating, werner2024cross}. 

Adaptive strategy selection methods address this problem by choosing a single strategy $\mathcal{A}^t$ at each round $t$ out of a portfolio $\mathcal{P} = \{\mathcal{A}_1, \ldots, \mathcal{A}_K\}$. This may be naturally framed as a multi-armed bandit (MAB) problem~\citep{baram2004online, hsu2015active, pang2018dynamic} in which each strategy is one arm of the MAB, and at each round the learner observes a reward from the previous round (e.g., accuracy on a validation dataset) and selects the next arm using a bandit solver. Strategy-wise performance may also be estimated by retraining the model after removing a subset of training points selected according to that strategy's own criterion~\citep{hacohen2024select}.

Both approaches are, however, poorly suited to the design of high-throughput biological experiments~\citep{rood2024toward,huang2024sequential,yang2025active,demeo2025active}, where assays efficiently provide results for thousands of inputs at once, but experimental time remains dominated by biological timescales. In these settings, with few rounds (small $T$) and large batch sizes (large $B$), relying on a single strategy at each round is unnecessarily risky.

We instead introduce the problem of \emph{soft strategy selection}. In this setting, we allocate fractions $p_k^t$ of each round's budget $B$ to all strategies in the portfolio $\mathcal{P}$, where the vector $\mathbf{p}^t \in \Delta^{K-1}$ lies in the $(K-1)$-dimensional probability simplex. This results in a batch $\mathcal{B}^t = \bigcup_{k=1}^K \mathcal{B}^t_k$, composed of strategy-specific sub-batches. We propose FractAL, the first method designed to address this problem. FractAL estimates a strategy-wise reward vector $\mathbf{r}^t$ by attributing changes in validation risk to individual strategies via influence functions~\citep{koh2017understanding}, and then updates the allocations using online mirror descent on the probability simplex.

\begin{wrapfigure}{r}{0.4\textwidth}
    \vspace{-\intextsep}
    \centering
    \includegraphics[width=0.9\linewidth]{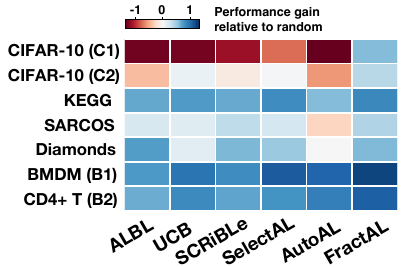}
    \caption{Performance gain over random sampling for the six adaptive strategy selection methods, across the seven setups of Section~\ref{sec:expt} (Appendix~\ref{app:perf-setups}).}
    \label{fig:mean_improv}
    \vspace{-1cm}
\end{wrapfigure}
After introducing the necessary background (Section~\ref{sec:bg}), we present FractAL in detail (Section~\ref{sec:method}). We demonstrate that FractAL surpasses bandit- and retraining-based baselines on multiple classification and regression problems, including recent high-throughput genetic perturbation screening datasets (Figure~\ref{fig:mean_improv}, Section~\ref{sec:expt}). These experiments also highlight that FractAL concentrates budget on the strongest strategy in each portfolio while pruning the weakest. Moreover, it is the only strategy selection method that outperforms random selection across all datasets (Figure~\ref{fig:mean_improv}), suggesting that fractional budget allocation is key to improving AL performance.

\section{Background}
\label{sec:bg}
\begin{table}[t]
    \caption{Comparison with related work on three axes: support for batch sampling, fractional budget allocation across the strategy portfolio, and the ability to operate without additional model training.
    }
    \label{tab:method_comparison}
    \centering
    \small
    \definecolor{mygreen}{RGB}{0,130,0}
    \definecolor{Gray}{gray}{0.92}
    \begin{tabular}{l ccc}
        \toprule
        Related work & Batch Sampling & \makecell{Fractional Budget\\Allocation} & \makecell{No Retraining /\\Aux. Models} \\
        \midrule
\textbf{ALBL}~{\footnotesize\citep{hsu2015active}}\cellcolor{Gray}     & \cellcolor{Gray}{\color{red}\large\ding{55}}    & \cellcolor{Gray}{\color{red}\large\ding{55}}    & \cellcolor{Gray}{\color{mygreen}\large\ding{51}} \\
        \textbf{AutoAL}~{\footnotesize\citep{wang2024autoal}}& {\color{mygreen}\large\ding{51}}& {\color{red}\large\ding{55}}    & {\color{red}\large\ding{55}} \\
        \textbf{SelectAL}~{\footnotesize\citep{hacohen2024select}}\cellcolor{Gray}  & \cellcolor{Gray}{\color{mygreen}\large\ding{51}}                & \cellcolor{Gray}{\color{red}\large\ding{55}}                    & \cellcolor{Gray}{\color{red}\large\ding{55}} \\
        \textbf{FractAL} \emph{(this work)}        & {\color{mygreen}\large\ding{51}}& {\color{mygreen}\large\ding{51}}& {\color{mygreen}\large\ding{51}} \\
        \bottomrule
    \end{tabular}
\end{table}

FractAL builds on two distinct lines of work: adaptive strategy selection for active learning and influence functions for data attribution. A comparison with existing adaptive strategy selection methods appears in Table~\ref{tab:method_comparison}.

\subsection{Adaptive Strategy Selection for Active Learning}
\label{bg:adapt}
Adaptive strategy selection has been well studied, and existing work falls broadly into three paradigms: bandit-based, retraining-based, and differentiable search methods.

\paragraph{Bandit-Based Strategy Selection.} Early approaches formulate strategy selection as a multi-armed bandit (MAB) problem where AL strategies act as experts~\citep{hsu2015active, baram2004online, pang2018dynamic}. For example, ALBL~\citep{hsu2015active} relies on the Exp4.P solver~\citep{beygelzimer2011contextual} with importance-weighted accuracy rewards. All these methods select a single point per iteration; in the batch setting, this amounts to committing each round to a single strategy, preventing strategy-level reward attribution within a batch.

\paragraph{Retraining-Based Strategy Selection.} SelectAL~\citep{hacohen2024select} uses a derivative-based surrogate test to switch between predefined \textit{low-budget} and \textit{high-budget} strategy families. At each round, each strategy selects points to remove from the training set, and the strategy yielding the largest risk increase after retraining is selected for the next round. The original specification of SelectAL is limited to $K=2$ strategies, but its main selection principle generalizes to any portfolio size, so we adopt it as a baseline. Evaluating $K$ strategies over $T$ AL rounds requires $O(KT)$ model retrainings, which becomes expensive when the surrogate model $f_\theta$ is costly to train.

\paragraph{Differentiable Strategy Search.} AutoAL~\citep{wang2024autoal} makes strategy selection differentiable via bi-level optimization with two co-trained auxiliary networks (SearchNet and FitNet). The bi-level optimization, however, introduces a substantial number of hyperparameters and adds computational overhead per iteration.

\subsection{Influence Functions}
\label{bg:if}
Influence functions~\citep{hampel1974influence} quantify the impact of individual datapoints on a model's performance. A datapoint's influence estimates how the validation loss would change if that point were removed from the training set. This is done by first approximating the resulting perturbation to the model parameters, and then propagating it to the validation loss via a first-order Taylor expansion around the converged parameters~\citep{koh2017understanding}.

Influence functions have previously been used as forward-looking AL acquisition heuristics, by estimating the influence at unmeasured inputs $x \in \mathcal{D}_\text{pool}$ and selecting the highest-ranked candidates~\citep{liu2021influence, xia2023reliable}. This approach, however, requires the construction of a pseudo-label for each unmeasured input (e.g., enumerating all possible labels, or training a different surrogate model)~\citep{xu2019understanding}, which makes it impractical for single- and multi-task regression settings. Here, we instead estimate per-strategy rewards using \emph{retrospective} influence of training points $x \in \mathcal{D}_\text{train}$ whose outputs have already been measured.

\section{FractAL: Fractional Active Learning}
\label{sec:method}
FractAL has two main components. We first define a strategy-wise reward signal, computed via influence-based attribution instead of retraining (Section~\ref{sec:str-infl-our}). Based on this signal, we then update the allocation $\mathbf{p}^{t+1}$ via online mirror descent on the probability simplex $\Delta^{K-1}$ (Section~\ref{sec:our-alg}).

\subsection{Strategy-wise Influence Attribution}
\label{sec:str-infl-our}
At round $t \in [T]$, our training data $\mathcal{D}^t_\text{train} = \bigcup_{t' = 0}^t \left(x_i, y_i \right)_{i \in \mathcal{B}^{t'}}$ is composed of the inputs and outputs for the union of all batches up to round $t$, where each batch $\mathcal{B}^t = \bigcup_{k=1}^K \mathcal{B}^t_k$ is composed of sub-batches $\mathcal{B}^t_k$ of size $n_k^t = |\mathcal{B}^t_k|$ selected by strategy $\mathcal{A}_k$. After updating the model parameters $\hat{\theta}^t$ by training the model $f_\theta$ on $\mathcal{D}^t_\text{train}$, we assess the contribution of each strategy's sub-batch $\mathcal{B}_k^t$ to the validation risk.
For $\boldsymbol{\epsilon} \in \mathbb{R}^K$, we define the perturbed training risk:
\begin{equation}
     \mathcal{F}_t(\boldsymbol{\epsilon}, \theta) := \mathcal{R}(\mathcal{D}^t_\text{train}; \theta) + \sum_{k=1}^K\frac{\epsilon_k}{n_k^t} \sum_{(x, y) \in \mathcal{B}^t_k} \mathcal{L}\left((x, y); \theta\right), \label{eq:perturb_risk}
\end{equation}
in which each coefficient $\epsilon_k$ reweights every point acquired by strategy $\mathcal{A}_k$. We denote the minimizer of the perturbed risk as $\hat{\theta}^t_{\boldsymbol{\epsilon}} := \arg\min_\theta \mathcal{F}_t(\boldsymbol{\epsilon}, \theta)$ so that $\hat{\theta}^t_{\boldsymbol{0}} = \hat{\theta}^t$ recovers the unperturbed parameters. 
The \textit{strategy-wise influence vector} $\mathcal{I}( \hat{\theta}^t) \in \mathbb{R}^K$ is defined as the gradient with respect to $\boldsymbol{\epsilon}$ of the validation risk evaluated at the perturbed parameters, taken at $\boldsymbol{\epsilon} = \mathbf{0}$. Its $k$-th component $\mathcal{I}_k(\hat{\theta}^t)$  thus
quantifies the first-order change in validation risk under upweighting of $\mathcal{A}_k$'s sub-batch. Under standard regularity assumptions on the loss function and the optimality of $\hat{\theta}^t$, it admits a closed-form expression:
\begin{restatable}[Strategy-wise Influence Decomposition]{prop}{strategyinfluence}
\label{lem:strategy-influence}
Assume the loss $\mathcal{L}$ is twice continuously differentiable in $\theta$ near $\hat\theta^t$, that $\hat\theta^t$ minimizes the training risk, and that the Hessian
$H_{\hat\theta^t} = \nabla^2_\theta \mathcal{R}(\mathcal{D}^t_{\text{train}};
\hat\theta^t)$ is positive definite. Then the $k$-th component of the influence vector is given by:
\begin{equation}
\label{eq:infl-cal-our}
    \mathcal{I}_k(\hat{\theta}^t) = -\nabla_{\theta} \mathcal{R}(\mathcal{D}_{\text{val}};\hat{\theta}^t)^\top H_{\hat{\theta}^t}^{-1} \left( \frac{1}{n^t_k} \sum_{(x, y) \in \mathcal{B}^t_k} \nabla_\theta \mathcal{L}\left((x, y); \hat{\theta}^t\right) \right).
\end{equation}
\end{restatable}
The proof appears in Appendix~\ref{app:proof}. We define the \textit{reward} of strategy $k$ at round $t$ as $r_k^t \triangleq -\mathcal{I}_k(\hat{\theta}^t)$. The sign of $r_k^t$ indicates whether $\mathcal{A}_k$'s sub-batch reduced ($r_k^t > 0$) or increased ($r_k^t < 0$) validation risk, while the magnitude reflects the size of the effect.

\paragraph{Implementation details.} 
The inverse Hessian-vector product $\mathbf{v}^t = \nabla_{\theta}\mathcal{R}(\mathcal{D}_{\text{val}};\hat{\theta}^t)^\top H_{\hat{\theta}^t}^{-1}$ from Eq.~\ref{eq:infl-cal-our} is independent of $k$, so we compute it once per round via the conjugate gradient method~\citep{martens2010deep} and then obtain each $r_k^t$ via a single inner product. When $f_\theta$ is parameterized as a deep neural network, we compute the Hessian with respect to the last layer parameters only~\citep{koh2017understanding} and regularize it using a damping term ($\lambda=0.01$). Both measures help ensure positive definiteness (see Appendix~\ref{app-prac-est}).

\subsection{Fractional Portfolio Allocation}
\label{sec:our-alg}

Raw reward values $\mathbf{r}^t$ can be noisy and lead to unstable allocation updates, particularly for strategies $\mathcal{A}_k$ where the sub-batch $\mathcal{B}_k^t$ contained few samples. We therefore use the smoothed reward vector $\mathbf{\bar{r}}^t$, defined as an exponentially weighted moving average over rounds:
\begin{equation}
    \bar{\mathbf{r}}^t = (1 - \alpha) \bar{\mathbf{r}}^{t-1} + \alpha\, \mathbf{r}^t,
\end{equation}
with $\bar{\mathbf{r}}^0 = \mathbf{0}$, and $\alpha = 0.5$ throughout our experiments. 
An intermediate allocation $\mathbf{q}^{t+1}$ is defined as the minimizer of the KL-regularized negative reward: 
\begin{equation}
\label{eq:opt}
    \mathbf{q}^{t+1} = \arg\min_{\mathbf{q} \in \Delta^{K-1}} \left[-\eta^t \langle \bar{\mathbf{r}}^t, \mathbf{q}\rangle + D_\mathrm{KL}(\mathbf{q} \,\|\, \mathbf{p}^t)\right],
\end{equation}
where $\eta^t = \eta_0 / \max_{k} |\bar{r}^t_{k}|$ is an adaptive step size, with $\eta_0=3$ across all experiments. This proximal-point form, standard in online convex optimization, prevents the allocation from changing too abruptly between rounds. Solving Eq.~\ref{eq:opt} yields the exponentiated gradient step:
\begin{equation}
    q_k^{t+1} = \frac{p_k^t \exp(\eta^t \bar{r}^t_k)}{\sum_{k'=1}^K p_{k'}^t \exp(\eta^t \bar{r}^t_{k'})}.
\end{equation}

The final allocation $\mathbf{p}^{t+1}$ mixes $\mathbf{q}^{t+1}$ with the
uniform distribution over the $K$ strategies, using mixing weight
$Kp_\mathrm{min}$:
\begin{equation}
\label{meth-pcal}
    p_k^{t+1} = (1 - Kp_\mathrm{min})\, q_k^{t+1} + p_\mathrm{min},
\end{equation}
where $p_\mathrm{min} = 1/B$ to guarantee that every strategy has a nonzero allocation. The allocation $\mathbf{p}^{t+1}$ then determines sub-batch sizes for the next round. Finally, strategies select their sub-batches sequentially in decreasing order of allocation, each excluding points already selected, so that the resulting sub-batches are disjoint (Algorithm~\ref{alg:FractAL}, Appendix~\ref{app:meth-detail}).

\begin{algorithm}[t]
\caption{FractAL: Fractional Allocation for Active Learning Strategy Portfolios}
\label{alg:FractAL}
\begin{algorithmic}[1]
\Require Strategy portfolio $\mathcal{P} = \{\mathcal{A}_1, \dots, \mathcal{A}_K\}$; pool $\mathcal{D}_\text{pool}$; initial labeled set $\mathcal{D}_\text{init}$; validation set $\mathcal{D}_\text{val}$; batch size $B$; rounds $T$.
\State Initialize $\mathcal{D}^0_\text{train} \gets \mathcal{D}_\text{init}$, $\mathbf{p}^1 \gets (1/K, \dots, 1/K) \in \Delta^{K-1}$, $\bar{\mathbf{r}}^0 \gets \mathbf{0}$
\For{$t = 1, 2, \dots, T-1$}
\State Allocate integer budgets $n^t_k$ summing to $B$
\State Form disjoint sub-batches $\mathcal{B}^t_k$ using $\mathcal{A}_k$ \Comment{App.~\ref{app:meth-detail}}
    \State Form composite batch $\mathcal{B}^t = \bigsqcup_{k=1}^K \mathcal{B}^t_k$ and update $\mathcal{D}^t_\text{train} \gets \mathcal{D}^{t-1}_\text{train} \cup \mathcal{B}^t$
    \State Train $f_\theta$ on $\mathcal{D}^t_\text{train}$ to obtain $\hat{\theta}^t$
    \State Compute reward $r^t_k = -\mathcal{I}_k(\hat{\theta}^t)$ for $k = 1, \dots, K$ \Comment{Eq.~\ref{eq:infl-cal-our}}
    \State Smooth: $\bar{r}^t_k \gets (1-\alpha) \bar{r}^{t-1}_k + \alpha\, r^t_k$
    \State Exponentiated gradient step: $q_k^{t+1} \gets \frac{p_k^t \exp(\eta^t \bar{r}^t_k)}{\sum_{k'=1}^K p_{k'}^t \exp(\eta^t \bar{r}^t_{k'})}$
    \State Enforce exploration floor: $p_k^{t+1} \gets (1 - K p_\mathrm{min}) q_k^{t+1} + p_\mathrm{min}$
\EndFor
\end{algorithmic}
\end{algorithm}

\section{Experiments}
\label{sec:expt}

We evaluate FractAL against bandit-based and retraining-based strategy selection methods across three task types: classification (CIFAR-10), single-task regression (KEGG metabolic reaction network, SARCOS inverse dynamics, and diamond pricing), and multi-task regression (high-throughput genetic perturbation screens). The screening experiments cover two immune cell types: CD4+ T cells and bone-marrow-derived myeloid cells. Throughout these experiments, we report the test performance of the surrogate model. We also assess whether each selection method concentrates
budget on the strongest strategies, ranked by their performance when run
alone for the entire AL run, which we refer to as their
\emph{pure-strategy performance}. Unless mentioned otherwise, metrics are averaged over $5$ data splits and $10$ random seeds per split, with standard error reported over the seeds after averaging across the splits.

\paragraph{Strategy Portfolio.} Within each task type, all strategy selection methods share a common portfolio. For regression tasks, we use random selection (Random), CoreSet~\citep{sener2017active}, ProbCover~\citep{yehuda2022active}, and TypiClust~\citep{hacohen2022active} ($K=4$). For classification, we additionally include BADGE~\citep{ash2019deep} and three uncertainty-based strategies~\citep{settles2009active}: least-confidence (LeastConf), maximum-entropy (MaxEnt), and minimum-margin (MinMargin), for a total of $K=8$ strategies. Further details appear in Appendix~\ref{app:strat-pool}.

\paragraph{Baselines.}
We compare FractAL against five adaptive strategy
selection methods and one non-adaptive control. The first is ALBL~\citep{hsu2015active}, a bandit-based strategy selection method which uses the adversarial Exp4.P solver with the validation metric as feedback. The second uses the same bandit formulation but replaces the Exp4.P solver by the upper confidence bound algorithm (UCB). Third, we cast soft strategy allocation as a linear bandit problem and solve it with SCRiBLe, a self-concordant regularized linear bandit solver~\citep{abernethy2009competing}. Like FractAL, SCRiBLe produces soft allocations, but it observes only bandit feedback. It therefore isolates the contribution of influence-based attribution. Finally, we include the retraining-based method SelectAL~\citep{hacohen2024select} and the differentiable strategy selection method AutoAL~\citep{wang2024autoal}. As a non-adaptive control, we also benchmark a uniform allocation across the
portfolio (Equal Split). Specifications of all the baselines appear in Appendix~\ref{app:baselines}. We also report the performance of the best and worst individual strategies for each setup (in terms of pure-strategy performance). Since those strategies are unknown in advance, these results serve only as a reference range.

\paragraph{Allocation diagnostics.} Qualitatively different strategy allocations can yield similar final test performance, so test metrics alone do not reveal whether a method allocates budget sensibly. A natural desideratum for a soft strategy selection method is to concentrate budget on strategies that perform well and prune those that do not. We quantify this for each setup by first identifying the sets of best- and worst-performing strategies from their pure-strategy test performance (independently of any adaptive method). We then compute the overlap between each method's top-ranked strategies and the best-performing set (Best-Set Overlap), and between its bottom-ranked strategies and the worst-performing set (Worst-Set Overlap), where the number of top- (or bottom-) ranked strategies considered equals the size of the Best- (or Worst-)Set (see Appendix~\ref{app:precision_k}). We report these diagnostics for the three methods that assign an internal score to every strategy at each round (SCRiBLe, SelectAL, and FractAL).

\subsection{Classification Benchmark: CIFAR-10}
\label{clas:intro}
We evaluate FractAL on CIFAR-10~\citep{krizhevsky2009learning} using a ResNet-18 surrogate model~\citep{he2016deep} trained from scratch over $T=10$ rounds, with test accuracy as the performance metric. We follow the pre-processing and training protocol of \citet{hacohen2024select} and consider two portfolio configurations. Configuration 1 (C1) is a three-strategy portfolio --- CoreSet, Random, and ProbCover --- chosen to span weak, neutral, and strong strategies based on their pure-strategy performance (Table~\ref{tab:clas-all-pure}). This setting tests whether FractAL concentrates budget on the strong strategy and prunes the weak one. Configuration 2 (C2) deploys all eight strategies simultaneously. Because most of the eight strategies perform within one standard error of each other and of Random (Table~\ref{tab:clas-all-pure}), the selection problem in this setup is considerably harder. Across both configurations, we use $|\mathcal{D}_{\text{init}}| = B = 500$ and $|\mathcal{D}_{\text{val}}| = 1024$. Implementation details, including hyperparameter configurations and optimization protocols, appear in Appendix~\ref{app:cf10-mot}.

\begin{figure}[t]
    \centering
    \begin{subfigure}[t]{0.24\textwidth}
        \centering
        \includegraphics[width=\textwidth]{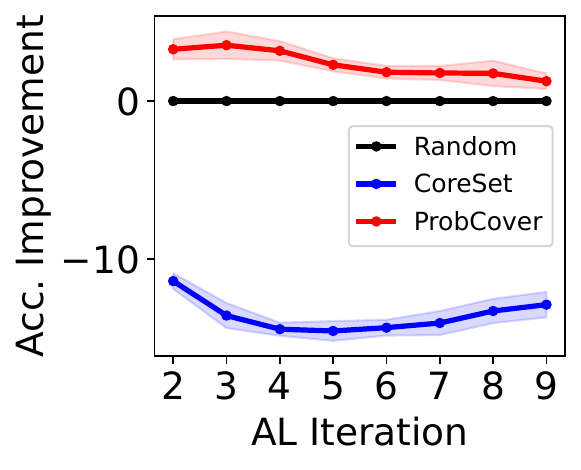}
        \caption{\centering Test performance of single strategies}
      \end{subfigure}
      \begin{subfigure}[t]{0.24\textwidth}
        \centering
        \includegraphics[width=\textwidth]{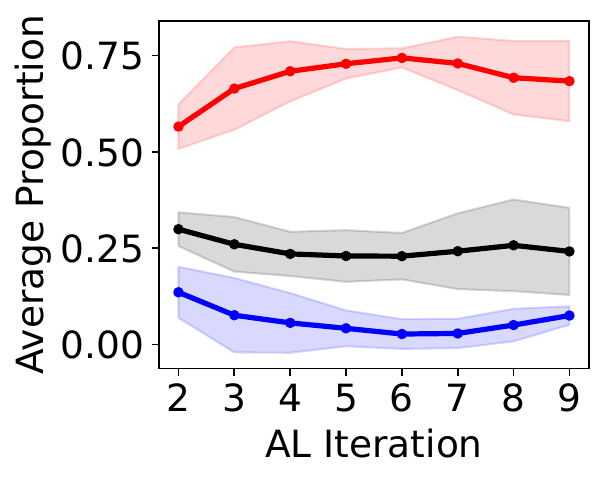}  
        \caption{\centering FractAL allocations}
      \end{subfigure}
      \begin{subfigure}[t]{0.24\textwidth}
        \centering
        \includegraphics[width=\textwidth]{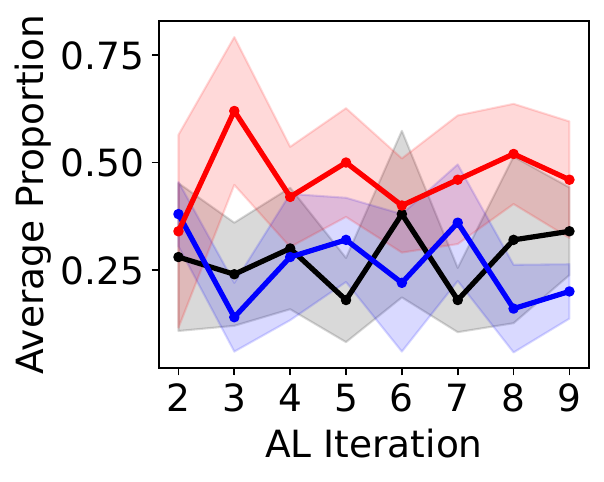}
        \caption{\centering SelectAL allocations}
      \end{subfigure}
      \begin{subfigure}[t]{0.24\textwidth}
        \centering
        \includegraphics[width=\textwidth]{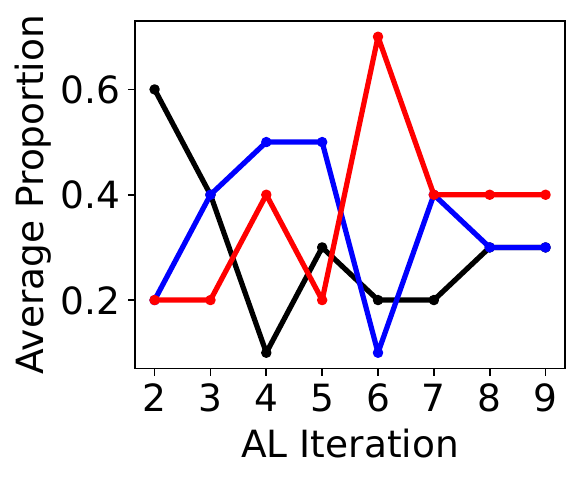}
        \caption{\centering ALBL allocations}
      \end{subfigure}
      \caption{Setup (C1): CIFAR-10. (a) Test accuracy of the three individual strategies. (b--d) Per-round average budget allocation under FractAL, SelectAL, and ALBL. Solid lines denote the mean and shaded bands indicate one standard deviation.
      }
\label{fig:frac_evolution}

\end{figure}

\begin{table}[t]
\centering
\begin{minipage}[t]{0.40\textwidth}
  \centering
  \small
  \caption{Test accuracy ($\uparrow$) on CIFAR-10.}
  \label{cifar-allcomp-results}
\begin{tabular}{lcc}
    \toprule
    Method      & (C1)  & (C2)\\
    \midrule
\multicolumn{3}{l}{\textit{Pure strategies (best/worst in hindsight)}}\\
\shortstack[l]{\textit{Best pure}\\ \phantom{\tiny x}}
  & \shortstack{\textit{71.2}\se{0.1}\\ \tiny\textit{ProbCover}}
  & \shortstack{\textit{71.2}\se{0.1}\\ \tiny\textit{ProbCover}}\\[2pt]
\shortstack[l]{\textit{Worst pure}\\ \phantom{\tiny x}}
  & \shortstack{\textit{58.7}\se{0.2}\\ \tiny\textit{CoreSet}}
  & \shortstack{\textit{58.7}\se{0.2}\\ \tiny\textit{CoreSet}}\\
    \midrule
    \multicolumn{3}{l}{\textit{Non-adaptive controls}}\\
    Random      & 69.4\se{0.1} & 69.4\se{0.1}\\
    Equal Split & 67.8\se{0.1} & 69.1\se{0.1}\\
    \midrule
    \multicolumn{3}{l}{\textit{Adaptive strategy selection methods}}\\
    ALBL        & 67.5\se{0.5} & 68.6\se{0.3}\\
    UCB         & 67.5\se{0.1} & 69.5\se{0.1}\\
    SCRiBLe     & 67.7\se{0.2} & 69.2\se{0.1}\\
    SelectAL    & 68.3\se{0.1} & 69.4\se{0.1}\\
    AutoAL      & 67.5\se{0.1} & 68.3\se{0.1}\\
    FractAL     & \textbf{70.2}\se{0.2} & \textbf{69.8}\se{0.2}\\
    \bottomrule
\end{tabular}
\end{minipage}\hfill
\begin{minipage}[t]{0.56\textwidth}
  \centering
  \footnotesize
  \caption{Test MSE ($\times 10^{-1}$; $\downarrow$) on regression datasets.}
  \label{reg-results}
  \setlength\tabcolsep{0.15cm}

  \begin{tabular}{lccc}
    \toprule
    Method & KEGG & SARCOS & Diamonds\\
    \midrule
    \multicolumn{4}{l}{\textit{Pure strategies (best/worst in hindsight)}}\\
    \shortstack[l]{\textit{Best pure}\\ \phantom{\tiny x}}
      & \shortstack{\textit{0.49}\se{0.00}\\ \tiny \textit{ProbCover}}
      & \shortstack{\textit{0.72}\se{0.01}\\ \tiny \textit{ProbCover}}
      & \shortstack{\textit{0.08}\se{0.01}\\ \tiny \textit{CoreSet}}\\[2pt]
    \shortstack[l]{\textit{Worst pure}\\ \phantom{\tiny x}}
      & \shortstack{\textit{0.97}\se{0.01}\\ \tiny \textit{TypiClust}}
      & \shortstack{\textit{0.79}\se{0.01}\\ \tiny \textit{Random}}
      & \shortstack{\textit{0.12}\se{0.01}\\ \tiny \textit{ProbCover}}\\
    \midrule
    \midrule
    \multicolumn{4}{l}{\textit{Non-adaptive controls}}\\
    Random      & 0.79\se{0.02} & 0.79\se{0.01} & 0.11\se{0.01}\\
    Equal Split & 0.62\se{0.01} & 0.76\se{0.01} & \textbf{0.08}\se{0.01}\\
    \midrule
    \multicolumn{4}{l}{\textit{Adaptive strategy selection methods}}\\
    ALBL        & 0.63\se{0.02} & 0.77\se{0.01} & 0.09\se{0.01}\\
    UCB         & 0.61\se{0.01} & 0.77\se{0.01} & 0.10\se{0.01}\\
    SCRiBLe     & 0.63\se{0.01} & 0.75\se{0.01} & 0.09\se{0.01}\\
    SelectAL    & 0.59\se{0.01} & 0.76\se{0.01} & 0.09\se{0.01}\\
    AutoAL      & 0.65\se{0.01} & 0.82\se{0.01} & 0.10\se{0.01}\\
    FractAL     & \textbf{0.58}\se{0.01} & \textbf{0.74}\se{0.01} & \textbf{0.08}\se{0.01}\\
    \bottomrule
\end{tabular}
\end{minipage}
% \vspace{-5mm}
\end{table}

FractAL achieves the highest test accuracy (70.2 $\pm$ 0.2) in C1 among selection methods and is the only one to surpass the Random baseline (Table~\ref{cifar-allcomp-results}). The allocation dynamics in Figure~\ref{fig:frac_evolution}b-d explain this gap: FractAL down-weights CoreSet's share toward zero and concentrates budget on ProbCover, while ALBL oscillates between strategies and SelectAL allocates roughly equal shares to CoreSet and Random, despite the substantial gap in their pure-strategy performance. FractAL leads both allocation diagnostics by a wide margin (Table~\ref{tab:recall-all}, Welch's t-test against the next best baseline, $p < 0.05$ for both). SCRiBLe, which shares FractAL's soft allocation but observes only validation feedback, does no better than a uniform allocation. Neither scalar validation feedback nor deletion-based scoring separates the strategies as cleanly as influence-based attribution.

In the harder C2 setup, the strategy selection methods are separated by much smaller performance margins. FractAL is still highest among adaptive methods (69.8 $\pm$ 0.2) and is the only one to significantly exceed Random performance (Welch's t-test, $p < 0.05$). The allocation diagnostics are more informative here. FractAL remains well above chance on both (one-sample t-test against chance, $p < 0.05$), with the larger gap on Worst-Set Overlap reflecting its consistent pruning of CoreSet to zero (Table~\ref{tab:recall-all}, Figure~\ref{fig:cf10_8comp_frac_evolution}b), while the closest baseline barely clears chance on either. This setup shows that FractAL remains effective with a larger portfolio. 

\subsection{Regression Benchmarks}
We evaluate FractAL on three regression datasets: KEGG metabolic reaction network~\citep{naeem2011kegg}, SARCOS inverse dynamics~\citep{vijayakumar2000locally}, and diamond pricing (Diamonds; \citealp{vanschoren2014openml}), with input dimensions $27$, $21$, and $29$ respectively. We follow the pre-processing of \citet{holzmuller2023framework} and use a three-layer multi-layer perceptron (MLP) as the surrogate model. We report mean squared error (MSE) on the held-out test set. All methods use $|\mathcal{D}_{\text{init}}| = B = 500$ and $|\mathcal{D}_{\text{val}}| = 1024$ over $T=10$ rounds. Implementation details appear in Appendix~\ref{app:reg-setup}.

FractAL achieves the lowest test MSE among selection methods on all three regression datasets (Table~\ref{reg-results}). FractAL also leads both allocation diagnostics on KEGG and SARCOS, with a tie on Diamonds (Table~\ref{tab:recall-all}). On Diamonds, the margins are narrow because the pure strategies perform similarly (Table~\ref{tab:reg-all-pure}), leaving little room to improve from uniform allocation. Allocations from bandit baselines either oscillate (KEGG, SARCOS) or remain near uniform (Diamonds), never converging on a stable allocation (Figures~\ref{fig:kegg_frac_evolution}--\ref{fig:dia_frac_evolution}). This suggests that their feedback is too noisy to effectively discriminate between strategies in this setup.

These datasets also illustrate the failure mode that motivates adaptive strategy selection. CoreSet, which performed the worst on CIFAR-10 (Table~\ref{tab:clas-all-pure}), is among the strongest on all three regression setups (Table~\ref{tab:reg-all-pure}). Conversely, ProbCover is strongest on KEGG and SARCOS but weakest on Diamonds. No fixed strategy is safe across these setups, which emphasizes the need for adaptive selection.

\begin{table}[t]
\centering
\caption{Best-Set Overlap ($\uparrow$) and Worst-Set Overlap ($\uparrow$) across all
seven setups.
\textit{Chance} is the expected overlap under a uniform
allocation, $|\text{set}|/K$ (set cardinalities in Table~\ref{tab:recall-k}).}
\label{tab:recall-all}
\setlength{\tabcolsep}{4pt}
\footnotesize
\begin{tabular}{l cc cc cc cc}
\toprule
 & \multicolumn{2}{c}{\textit{Chance}} & \multicolumn{2}{c}{SCRiBLe}
 & \multicolumn{2}{c}{SelectAL} & \multicolumn{2}{c}{FractAL} \\
\cmidrule(lr){2-3}\cmidrule(lr){4-5}\cmidrule(lr){6-7}\cmidrule(lr){8-9}
Setup & \makecell{Best-Set\\Overlap} & \makecell{Worst-Set\\Overlap}
      & \makecell{Best-Set\\Overlap} & \makecell{Worst-Set\\Overlap}
      & \makecell{Best-Set\\Overlap} & \makecell{Worst-Set\\Overlap}
      & \makecell{Best-Set\\Overlap} & \makecell{Worst-Set\\Overlap} \\\midrule
CIFAR-10 (C1) & 0.33 & 0.33 & 0.34\se{0.03} & 0.35\se{0.05} & 0.46\se{0.02} & 0.41\se{0.03} & \textbf{0.79}\se{0.02} & \textbf{0.89}\se{0.02} \\
CIFAR-10 (C2) & 0.25 & 0.13 & 0.24\se{0.01} & 0.10\se{0.02} & 0.29\se{0.02} & 0.17\se{0.02} & \textbf{0.46}\se{0.01} & \textbf{0.61}\se{0.01} \\
KEGG          & 0.50 & 0.25 & 0.51\se{0.01} & 0.29\se{0.03} & 0.50\se{0.01} & 0.24\se{0.01} & \textbf{0.58}\se{0.01} & \textbf{0.31}\se{0.02} \\
SARCOS        & 0.50 & 0.75 & 0.51\se{0.01} & 0.75\se{0.01} & 0.53\se{0.02} & 0.77\se{0.01} & \textbf{0.54}\se{0.01} & \textbf{0.79}\se{0.01} \\
Diamonds      & 0.50 & 0.50 & \textbf{0.52}\se{0.02} & \textbf{0.52}\se{0.01} & 0.51\se{0.01} & 0.51\se{0.01} & \textbf{0.52}\se{0.02} & \textbf{0.52}\se{0.02} \\
BMDMs (B1)    & 0.25 & 0.50 & 0.26\se{0.02} & 0.50\se{0.01} & 0.59\se{0.03} & 0.56\se{0.01} & \textbf{0.70}\se{0.10} & \textbf{0.62}\se{0.01} \\
T Cells (B2)   & 0.25 & 0.50 & 0.20\se{0.03} & 0.50\se{0.01} & 0.25\se{0.01} & 0.49\se{0.00} & \textbf{0.57}\se{0.12} & \textbf{0.61}\se{0.05} \\
\bottomrule
\end{tabular}
\end{table}

\subsection{Design of High-Throughput Genetic Perturbation Screens}
\label{sec:bio_application}
Pooled genetic screens such as Perturb-seq, with a single-cell transcriptomics readout~\citep{dixit2016perturb}, and Optical Pooled Screens (OPS), with an imaging readout~\citep{feldman2019optical}, make it possible to alter the function of thousands of genes in a single experiment and measure their effects on cells. These data are key to training predictive models of cellular responses to perturbations across biological contexts, which is central to inferring gene regulatory circuits, identifying therapeutic targets, and engineering cellular phenotypes~\citep{rood2024toward, roohani2024predicting}. A recurring difficulty is that perturbation effects measured under one set of conditions (e.g., a different modality, timepoint, or cellular context) may not generalize well to another~\citep{dibaeinia2026virtual}. 

This motivates an experimental design problem where one chooses which perturbations to profile in order to maximize the predictive accuracy of the resulting model. This is a well-defined active learning task~\citep{huang2024sequential}. Because each round of wet-lab experimentation takes weeks but provides measurements for hundreds of perturbations in parallel, this setting also matches FractAL's operating regime~\citep{demeo2025active}. We evaluate FractAL on two practical instances of perturbation prediction, each testing a different aspect of generalization: cross-modal prediction from imaging to transcriptomic data (B1), and cross-temporal prediction within a single cell type (B2).

\paragraph{Datasets.} For the cross-modal task (B1), we rely on matched Perturb-seq and OPS datasets of primary mouse bone-marrow-derived myeloid cells (BMDMs) under immune stimulation~\citep{kudo2026scalable}, retaining the 514 perturbations with significant effects in both modalities. The goal is to predict each perturbation's effect on gene expression (quantified as the log-fold change between perturbed and control cells;~\citealp{boyeau2023empirical}) from embeddings of the imaging data. For the cross-temporal task (B2), we rely on a genome-scale Perturb-seq screen of primary human CD4+ T cells profiled at multiple time points after T cell receptor stimulation~\citep{zhu2025genome}, retaining the 709 perturbations with significant effects at both time points. The goal is to predict the late-stage transcriptional response (48 hours post-stimulation) of each perturbation from observations of its early-stage response (8 hours post-stimulation), testing the model's ability to extrapolate across the temporal dynamics of T cell activation. Data processing details for both tasks appear in Appendix~\ref{app:pert-pair}.

\begin{wraptable}{R}{0.43\textwidth}
% \vspace{-0.4cm}
\caption{AUPRC ($\uparrow$) for perturbation prediction tasks.
}
% \vspace{-0.3cm}
\footnotesize
  \label{pp-results}
  \centering
  \tabcolsep=0.11cm
\begin{tabular}{lcc}
    \toprule
    Method  & (B1): BMDMs  & (B2): T cells\\[0.1cm]
    \midrule
    \multicolumn{3}{l}{\textit{Pure strategies (best/worst in hindsight)}}\\
    \shortstack[l]{\textit{Best pure}\\ \phantom{\tiny x}}
      & \shortstack{\textit{0.4377}\se{0.0005}\\ \tiny \textit{CoreSet}}
      & \shortstack{\textit{0.3696}\se{0.0002}\\ \tiny \textit{CoreSet}}\\[2pt]
    \shortstack[l]{\textit{Worst pure}\\ \phantom{\tiny x}}
      & \shortstack{\textit{0.4165}\se{0.0011}\\ \tiny \textit{Random}}
      & \shortstack{\textit{0.3403}\se{0.0045}\\ \tiny \textit{Random}}\\
    \midrule
    \multicolumn{3}{l}{\textit{Non-adaptive controls}}\\
    Random      & 0.4165\se{0.0011} & 0.3403\se{0.0045}\\
    Equal Split & 0.4279\se{0.0006} & 0.3442\se{0.0065}\\
    \midrule
    \multicolumn{3}{l}{\textit{Adaptive strategy selection methods}}\\
    ALBL        & 0.4289\se{0.0018} & 0.3548\se{0.0021}\\
    UCB         & 0.4323\se{0.0013} & 0.3590\se{0.0015}\\
    SCRiBLe     & 0.4299\se{0.0006} & 0.3557\se{0.0006}\\
    SelectAL    & 0.4346\se{0.0010} & 0.3579\se{0.0012}\\
    AutoAL      & 0.4336\se{0.0005} & 0.3603\se{0.0008}\\
    FractAL     & \textbf{0.4363}\se{0.0006} & \textbf{0.3642}\se{0.0013}\\
    \bottomrule
\end{tabular}
\end{wraptable}

\paragraph{Experimental setup.}
Across both tasks, we use an initial labeled set of size $|\mathcal{D}_{\text{init}}| = 35$, batch size $B = 35$, and a validation set of size $|\mathcal{D}_{\text{val}}| = 40$. The held-out test sets are of size $|\mathcal{D}_{\text{test}}| = 103$ (B1) and $141$ (B2), corresponding to roughly 20\% of each dataset. Batch sizes are small because paired measurements of this kind are recent, and the resulting datasets are still limited in size. We run active learning for $T=6$ rounds (measuring roughly 40--50\% of the total pool) using a single-layer MLP as the surrogate model. Classical regression metrics such as MSE are known to be misaligned with the downstream task of identifying which genes respond to a perturbation~\citep{zhu2025auprc, littman2025gene}. We therefore report the area under the precision-recall curve (AUPRC) for detecting differentially expressed genes. The portfolio comprises Random, CoreSet, ProbCover, and TypiClust.

\paragraph{Results.} 
FractAL achieves the highest test AUPRC among adaptive methods on both perturbation prediction tasks (Table~\ref{pp-results}). It also leads every allocation diagnostic (Table~\ref{tab:recall-all}). FractAL steadily concentrates budget on the strongest strategy (here, CoreSet) on both tasks, whereas the bandit baselines (UCB, SCRiBLe, ALBL) never settle on a stable allocation (Figures~\ref{fig:tcell_frac_evolution_main},~\ref{fig:pp_frac_evolution}~and~\ref{fig:tcell_frac_evolution}). This behavior is consistent with the bandit feedback signal being too sparse to confidently commit budget within the few rounds available in this setting. FractAL also downweights the weakest strategy (Random), whereas the closest baseline (SelectAL) still retains substantial allocation to it. On B2, FractAL is the only method to significantly outperform chance on Worst-Set Overlap (one-sample t-test against chance, $p < 0.05$). This suggests the influence-based reward is a more reliable signal for pruning than SelectAL's deletion-based test.
    
\begin{figure}[t]
  \centering
  \begin{subfigure}[t]{0.24\textwidth}
    \centering
    \includegraphics[width=\textwidth]{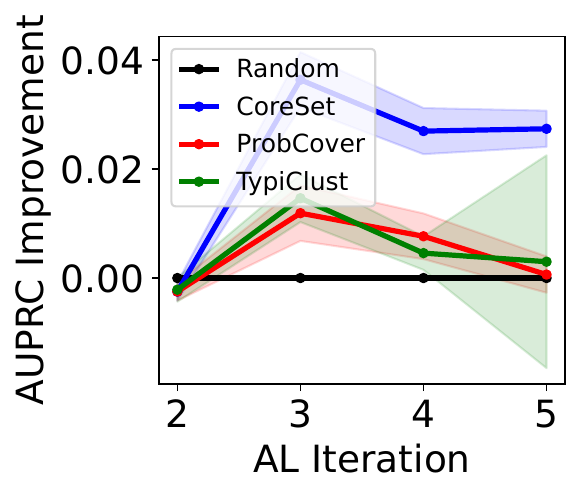}
    \caption{Test performance of single strategies}
  \end{subfigure}
  \hfill
  \begin{subfigure}[t]{0.24\textwidth}
    \centering
    \includegraphics[width=\textwidth]{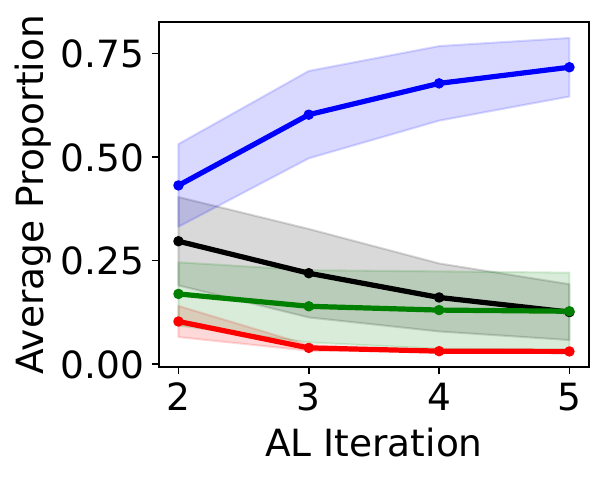}
    \caption{FractAL allocations}
  \end{subfigure}
  \begin{subfigure}[t]{0.24\textwidth}
    \centering
    \includegraphics[width=\textwidth]{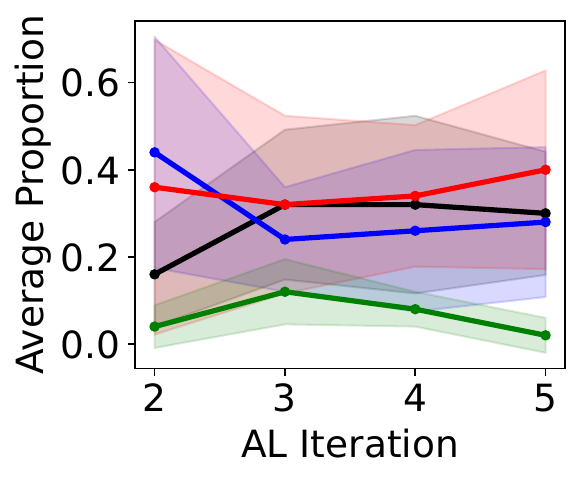}
    \caption{SelectAL allocations \centering}
  \end{subfigure}
  \begin{subfigure}[t]{0.24\textwidth}
    \centering
    \includegraphics[width=\textwidth]{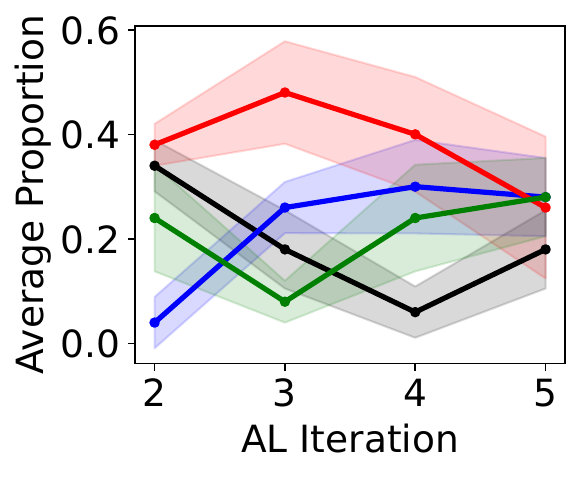}
    \caption{ALBL allocations \centering}
  \end{subfigure}
  \caption{Setup (B2): T cell perturbation prediction. (a) Test AUPRC of the four pure strategies. (b--d) Per-round average budget allocation under FractAL, SelectAL and ALBL. Solid lines denote the mean and shaded bands indicate one standard deviation.
  }
  \label{fig:tcell_frac_evolution_main}
\end{figure}

\subsection{Aggregate Results, Ablations, and Runtime}

\paragraph{Aggregate performance.} Metrics within a single experiment separate methods only by narrow margins. We therefore summarize performance across all seven setups, normalizing the metric so that $0$ corresponds to matching random performance and $1$ to matching the best strategy in the portfolio for that setup. FractAL attains the highest mean gain ($0.689$) and is the only method with a positive worst-case gain ($0.222$): every other method falls below Random on at least one setup. It is also the top-ranked method on average (mean rank $1.14$ of $7$, median $1$), ahead of the next closest method, SelectAL ($3.14$, median $3$). FractAL's advantage reflects consistent competitiveness across every task we evaluate
(Appendix~\ref{app:perf-setups}).

\paragraph{Ablation Studies.} We conduct a series of ablation studies to isolate the contribution of each component of FractAL to its performance across the three task domains (Appendix~\ref{app-ablation}). FractAL's performance is robust to the choice of the smoothing parameter $\alpha$. Mirror descent outperforms an additive update with clipping and renormalization. Soft fractional allocation improves over hard allocation (i.e., picking one strategy per round) on five out of seven setups while tying on the remaining two. Finally, the influence-based reward matches the performance of a retraining-based signal, and the two correlate well (Spearman's $\rho=0.7$). 

\paragraph{Runtime Analysis.}
\begin{wrapfigure}{r}{0.45\textwidth}
    \centering
    \includegraphics[width=0.9\linewidth]{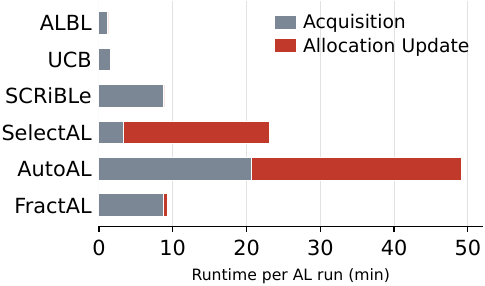}
    \caption{Runtime for an AL run on (C2).}
    \label{fig:runtime}
\end{wrapfigure}
We report the runtime for an entire active learning run on the CIFAR-10 dataset (C2, $K=8$) on a single NVIDIA L40S GPU (Figure~\ref{fig:runtime}, exact numbers in Appendix~\ref{app:runtime}). For FractAL and each of the baselines, we evaluate the time spent querying strategies in the portfolio and assembling the batch (Acquisition) versus the time spent on computing the reward signal and the allocation (Allocation Update). 
FractAL and SCRiBLe query every strategy each round, so their acquisition time is longer than that of the methods that commit to one at every round (ALBL, UCB, SelectAL).
Allocation-update time varies much more across methods: 0.5 min for FractAL, compared with 19.8 min for SelectAL and 28.5 min for AutoAL, which represents at least a 40-fold reduction compared to methods that involve additional model training. The runtime advantage of our method grows with both the portfolio size $K$ and the size of the surrogate model, since FractAL computes a single inverse-Hessian-vector product per round regardless of $K$, while retraining-based attribution scales linearly in both. The bandit baselines have a faster allocation update than FractAL, but they perform poorly in this study (Table~\ref{tab:recall-all}). 
Overall, FractAL completes an entire AL run $2.5\times$ faster than the second-best-performing method in the study (SelectAL), and $5\times$ faster than AutoAL.

\section{Discussion} 
\label{sec:disc}
We addressed the problem of strategy selection in active learning, where strategies perform unpredictably across tasks, datasets and model architectures. We proposed FractAL, a fractional allocation framework that acquires data from multiple strategies in parallel and updates their budget shares at each round using influence-based reward attribution. Aggregated across the seven setups, FractAL is the only method whose worst-case performance exceeds that of random sampling; every baseline falls below it on at least one setup (Appendix~\ref{app:perf-setups}). Across classification, regression, and high-throughput genetic perturbation screens, FractAL consistently up-weights the strongest strategy, and down-weights the ones with poor performance. In contrast, bandit-based baselines oscillate between strategies, and retraining-based methods underperform in this study. The allocation advantage is clearest in the iteration-scarce, large-batch regime that motivates the framework, where each round of feedback must carry substantial signal to be useful.

Several directions warrant further investigation. Influence functions are known to be fragile in deep networks~\citep{basu2020influence}, and alternative data valuation frameworks --- Shapley values~\citep{ghorbani2019data}, gradient similarity~\citep{pruthi2020estimating}, or Banzhaf values~\citep{wang2023data} --- may stabilize the reward signal under more complex surrogates. Beyond strategy selection, the fractional allocation framework extends to allocate budget across competing \emph{priors} (e.g., protein--protein interaction networks~\citep{szklarczyk2025string} or gene ontology annotations~\citep{aleksander2025gene}). Finally, we also see this problem as a promising direction for active experimental design more broadly.

% \newpage
\bibliography{iclr2027_conference}

@inproceedings{hsu2015active,
  title={Active Learning by Learning},
  author={Hsu, Wei-Ning and Lin, Hsuan-Tien},
  booktitle={AAAI Conference on Artificial Intelligence},
  pages={2659--2665},
  year={2015}
}

@InProceedings{wang2024autoal,
  title = 	 {{A}uto{AL}: Automated Active Learning with Differentiable Query Strategy Search},
  author =       {Wang, Yifeng and Zhan, Xueying and Huang, Siyu},
  booktitle = 	 {International Conference on Machine Learning},
  year = 	 {2025},
}

@inproceedings{hacohen2024select,
  title={How to Select Which Active Learning Strategy is Best Suited for Your Specific Problem and Budget},
  author={Hacohen, Guy and Weinshall, Daphna},
  booktitle={Advances in Neural Information Processing Systems},
  volume={36},
  year={2023}
}

@article{yang2025active,
  title={Active learning-assisted directed evolution},
  author={Yang, Jason and Lal, Ravi G and Bowden, James C and Astudillo, Raul and Hameedi, Mikhail A and Kaur, Sukhvinder and Hill, Matthew and Yue, Yisong and Arnold, Frances H},
  journal={Nature Communications},
  volume={16},
  number={1},
  pages={714},
  year={2025},
  publisher={Nature Publishing Group UK London}
}

@article{frey2025lab,
  title={Lab-in-the-loop therapeutic antibody design with deep learning},
  author={Frey, Nathan C and H{\"o}tzel, Isidro and Stanton, Samuel D and Kelly, Ryan and Alberstein, Robert G and Makowski, Emily K and Martinkus, Karolis and Berenberg, Daniel and Bevers III, Jack and Bryson, Tyler and others},
  journal={bioRxiv},
  pages={2025--02},
  year={2025},
  publisher={Cold Spring Harbor Laboratory}
}

@article{dixit2016perturb,
  title={{Perturb-Seq}: dissecting molecular circuits with scalable single-cell {RNA} profiling of pooled genetic screens},
  author={Dixit, Atray and Parnas, Oren and Li, Biyu and Chen, Jenny and Fulco, Charles P and Jerby-Arnon, Livnat and Marjanovic, Nemanja D and Dionne, Danielle and Burks, Tyler and Raychowdhury, Raktima and others},
  journal={Cell},
  volume={167},
  number={7},
  pages={1853--1866},
  year={2016},
}

@article{littman2025gene,
  title={Gene-embedding-based prediction and functional evaluation of perturbation expression responses with presage},
  author={Littman, Russell and Levine, Jacob and Maleki, Sepideh and Lee, Yongju and Ermakov, Vladimir and Qiu, Lin and Wu, Alexander and Huang, Kexin and Lopez, Romain and Scalia, Gabriele and others},
  journal={bioRxiv},
  pages={2025--06},
  year={2025},
  publisher={Cold Spring Harbor Laboratory}
}

@article{zhu2025auprc,
  title={{AUPRC}: a metric for evaluating the performance of in-silico perturbation methods in identifying differentially expressed genes},
  author={Zhu, Hongxu and Asiaee, Amir and Azinfar, Leila and Li, Jun and Liang, Han and Irajizad, Ehsan and Do, Kim-Anh and Long, James P},
  journal={Briefings in bioinformatics},
  volume={26},
  number={5},
  pages={bbaf426},
  year={2025},
  publisher={Oxford University Press}
}

@article{carpenter2006cellprofiler,
  title={{CellProfiler}: image analysis software for identifying and quantifying cell phenotypes},
  author={Carpenter, Anne E and Jones, Thouis R and Lamprecht, Michael R and Clarke, Colin and Kang, In Han and Friman, Ola and Guertin, David A and Chang, Joo Han and Lindquist, Robert A and Moffat, Jason and others},
  journal={Genome Biology},
  volume={7},
  number={10},
  pages={R100},
  year={2006},
  publisher={Springer}
}

@article{boyeau2023empirical,
  title={An empirical {Bayes} method for differential expression analysis of single cells with deep generative models},
  author={Boyeau, Pierre and Regier, Jeffrey and Gayoso, Adam and Jordan, Michael I and Lopez, Romain and Yosef, Nir},
  journal={Proceedings of the National Academy of Sciences},
  volume={120},
  number={21},
  pages={e2209124120},
  year={2023},
  publisher={National Academy of Sciences}
}

@article{dibaeinia2026virtual,
  title={Virtual Cells Need Context, Not Just Scale},
  author={Dibaeinia, Payam and Babu, Sudarshan and Knudson, Mei and ElSheikh, Ali and Wen, Yibo and Liu, Han and Perera, Jason and Khan, Aly A},
  journal={International Conference on Machine Learning},
  year={2026},
}

@article{feldman2019optical,
  title={Optical pooled screens in human cells},
  author={Feldman, David and Singh, Avtar and Schmid-Burgk, Jonathan L and Carlson, Rebecca J and Mezger, Anja and Garrity, Anthony J and Zhang, Feng and Blainey, Paul C},
  journal={Cell},
  volume={179},
  number={3},
  pages={787--799},
  year={2019},
  publisher={Elsevier}
}

@misc{naeem2011kegg,
  author = {Naeem, M. and Asghar, S.},
  title = {{KEGG} Metabolic Reaction Network (Undirected)},
  year = {2011},
  publisher = {UCI Machine Learning Repository},
  doi = {10.24432/C5G609},
}

@article{vanschoren2014openml,
  title={{OpenML}: networked science in machine learning},
  author={Vanschoren, Joaquin and Van Rijn, Jan N and Bischl, Bernd and Torgo, Luis},
  journal={ACM SIGKDD Explorations Newsletter},
  volume={15},
  number={2},
  pages={49--60},
  year={2014},
}

@article{houlsby2011bayesian,
  title={Bayesian active learning for classification and preference learning},
  author={Houlsby, Neil and Husz{\'a}r, Ferenc and Ghahramani, Zoubin and Lengyel, M{\'a}t{\'e}},
  journal={arXiv},
  year={2011}
}

@techreport{settles2009active,
  title={Active Learning Literature Survey},
  author={Settles, Burr},
  institution={University of Wisconsin-Madison},
  type={Computer Sciences Technical Report},
  number={1648},
  year={2009}
}

@inproceedings{sener2017active,
  title={Active Learning for Convolutional Neural Networks: A Core-Set Approach},
  author={Sener, Ozan and Savarese, Silvio},
  booktitle={International Conference on Learning Representations},
  year={2018}
}

@article{baram2004online,
  title={Online choice of active learning algorithms},
  author={Baram, Yoram and Yaniv, Ran El and Luz, Kobi},
  journal={Journal of Machine Learning Research},
  volume={5},
  pages={255--291},
  year={2004}
}

@inproceedings{beygelzimer2011contextual,
  title={Contextual bandit algorithms with supervised learning guarantees},
  author={Beygelzimer, Alina and Langford, John and Li, Lihong and Reyzin, Lev and Schapire, Robert},
  booktitle={International Conference on Artificial Intelligence and Statistics},
  pages={19--26},
  year={2011},
}

@inproceedings{pang2018dynamic,
  title={Dynamic ensemble active learning: A non-stationary bandit with expert advice},
  author={Pang, Kunkun and Dong, Mingzhi and Wu, Yang and Hospedales, Timothy M},
  booktitle={International Conference on Pattern Recognition},
  pages={2269--2276},
  year={2018},
}

@inproceedings{koh2017understanding,
  title={Understanding black-box predictions via influence functions},
  author={Koh, Pang Wei and Liang, Percy},
  booktitle={International Conference on Machine Learning},
  pages={1885--1894},
  year={2017},
}

@article{xu2019understanding,
  title={Understanding goal-oriented active learning via influence functions},
  author={Xu, Minjie and Kazantsev, Gary},
  journal={Workshop on ML with Guarantees, Advances in Neural Information Processing Systems},
  year={2019}
}

@article{hampel1974influence,
  title={The influence curve and its role in robust estimation},
  author={Hampel, Frank R},
  journal={Journal of the American Statistical Association},
  volume={69},
  number={346},
  pages={383--393},
  year={1974},
  publisher={Taylor \& Francis}
}

@inproceedings{martens2010deep,
  title={Deep learning via {Hessian-free} optimization.},
  author={Martens, James},
  booktitle={International Conference on Machine Learning},
  volume={27},
  pages={735--742},
  year={2010}
}

@article{xia2023reliable,
  title={Reliable active learning via influence functions},
  author={Xia, Meng and Henao, Ricardo},
  journal={Transactions on Machine Learning Research},
  volume={2023},
  pages={1364},
  year={2023}
}

@inproceedings{liu2021influence,
  title={Influence selection for active learning},
  author={Liu, Zhuoming and Ding, Hao and Zhong, Huaping and Li, Weijia and Dai, Jifeng and He, Conghui},
  booktitle={IEEE/CVF International Conference on Computer Vision},
  pages={9274--9283},
  year={2021}
}

@article{ash2019deep,
  title={Deep batch active learning by diverse, uncertain gradient lower bounds},
  author={Ash, Jordan T and Zhang, Chicheng and Krishnamurthy, Akshay and Langford, John and Agarwal, Alekh},
  journal={International Conference on Learning Representations},
  year={2020}
}

@article{yehuda2022active,
  title={Active learning through a covering lens},
  author={Yehuda, Ofer and Dekel, Avihu and Hacohen, Guy and Weinshall, Daphna},
  journal={Advances in Neural Information Processing Systems},
  volume={35},
  pages={22354--22367},
  year={2022}
}

@article{hacohen2022active,
  title={Active learning on a budget: Opposite strategies suit high and low budgets},
  author={Hacohen, Guy and Dekel, Avihu and Weinshall, Daphna},
  journal={International Conference on Machine Learning},
  year={2022}
}

@article{krizhevsky2009learning,
  title={Learning multiple layers of features from tiny images},
  author={Krizhevsky, Alex and Hinton, Geoffrey and others},
  journal={Tech Report},
  year={2009},
  publisher={Toronto, ON, Canada}
}

@article{holzmuller2023framework,
  title={A framework and benchmark for deep batch active learning for regression},
  author={Holzm{\"u}ller, David and Zaverkin, Viktor and K{\"a}stner, Johannes and Steinwart, Ingo},
  journal={Journal of Machine Learning Research},
  volume={24},
  number={164},
  pages={1--81},
  year={2023}
}

@inproceedings{he2016deep,
  title={Deep residual learning for image recognition},
  author={He, Kaiming and Zhang, Xiangyu and Ren, Shaoqing and Sun, Jian},
  booktitle={IEEE Conference on Computer Vision and Pattern Recognition},
  pages={770--778},
  year={2016}
}

@inproceedings{munjal2022towards,
  title={Towards robust and reproducible active learning using neural networks},
  author={Munjal, Prateek and Hayat, Nasir and Hayat, Munawar and Sourati, Jamshid and Khan, Shadab},
  booktitle={Proceedings of the IEEE/CVF Conference on Computer Vision and Pattern Recognition},
  pages={223--232},
  year={2022}
}

@inproceedings{abernethy2009competing,
  title={Competing in the dark: An efficient algorithm for bandit linear optimization},
  author={Abernethy, Jacob D and Hazan, Elad and Rakhlin, Alexander},
  booktitle={Conference on Learning Theory},
  year={2009}
}

@inproceedings{vijayakumar2000locally,
  title={Locally weighted projection regression: An {O}(n) algorithm for incremental real time learning in high dimensional space},
  author={Vijayakumar, Sethu and Schaal, Stefan},
  booktitle={International Conference on Machine Learning},
  volume={1},
  pages={288--293},
  year={2000},
}

@article{zhu2025genome,
  title={Genome-scale perturb-seq in primary human {CD4+ T} cells maps context-specific regulators of {T} cell programs and human immune traits},
  author={Zhu, Ronghui and Dann, Emma and Yan, Jun and Retana, Justine Reyes and Goto, Ryunosuke and Guitche, Reese C and Petersen, Lillian K and Ota, Mineto and Pritchard, Jonathan K and Marson, Alexander},
  journal={Cell},
  year={2026},
}

@article{smith2018less,
  title={Less is more: Sampling chemical space with active learning},
  author={Smith, Justin S and Nebgen, Ben and Lubbers, Nicholas and Isayev, Olexandr and Roitberg, Adrian E},
  journal={The Journal of Chemical Physics},
  volume={148},
  number={24},
  year={2018},
  publisher={AIP Publishing}
}

@inproceedings{huang2024sequential,
  title={Sequential optimal experimental design of perturbation screens guided by multi-modal priors},
  author={Huang, Kexin and Lopez, Romain and H{\"u}tter, Jan-Christian and Kudo, Takamasa and Rios, Antonio and Regev, Aviv},
  booktitle={International Conference on Research in Computational Molecular Biology},
  pages={17--37},
  year={2024},
}

@article{luth2023navigating,
  title={Navigating the pitfalls of active learning evaluation: A systematic framework for meaningful performance assessment},
  author={L{\"u}th, Carsten and Bungert, Till and Klein, Lukas and Jaeger, Paul},
  journal={Advances in Neural Information Processing Systems},
  volume={36},
  pages={9789--9836},
  year={2023}
}

@article{werner2024cross,
  title={A cross-domain benchmark for active learning},
  author={Werner, Thorben and Burchert, Johannes and Stubbemann, Maximilian and Schmidt-Thieme, Lars},
  journal={Advances in Neural Information Processing Systems},
  volume={37},
  pages={62875--62911},
  year={2024}
}

@inproceedings{wang2023data,
  title={Data banzhaf: A robust data valuation framework for machine learning},
  author={Wang, Jiachen T and Jia, Ruoxi},
  booktitle={International Conference on Artificial Intelligence and Statistics},
  pages={6388--6421},
  year={2023},
}

@article{pruthi2020estimating,
  title={Estimating training data influence by tracing gradient descent},
  author={Pruthi, Garima and Liu, Frederick and Kale, Satyen and Sundararajan, Mukund},
  journal={Advances in Neural Information Processing Systems},
  volume={33},
  pages={19920--19930},
  year={2020}
}

@inproceedings{ghorbani2019data,
  title={Data shapley: Equitable valuation of data for machine learning},
  author={Ghorbani, Amirata and Zou, James},
  booktitle={International Conference on Machine Learning},
  pages={2242--2251},
  year={2019},
}

@article{demeo2025active,
  title={Active learning framework leveraging transcriptomics identifies modulators of disease phenotypes},
  author={DeMeo, Benjamin and Nesbitt, Charlotte and Miller, Samuel A and Burkhardt, Daniel B and Lipchina, Inna and Fu, Doris and Holderrieth, Peter and Kim, David and Kolchenko, Sergey and Szalata, Artur and others},
  journal={Science},
  volume={390},
  number={6776},
  pages={8577},
  year={2025},
  publisher={American Association for the Advancement of Science}
}

@article{rood2024toward,
  title={Toward a foundation model of causal cell and tissue biology with a Perturbation Cell and Tissue Atlas},
  author={Rood, Jennifer E and Hupalowska, Anna and Regev, Aviv},
  journal={Cell},
  volume={187},
  number={17},
  pages={4520--4545},
  year={2024},
  publisher={Elsevier}
}

@article{tran2026rapid,
  title={Rapid directed evolution guided by protein language models and epistatic interactions},
  author={Tran, Vincent Q and Nemeth, Matthew and Bartie, Liam J and Chandrasekaran, Sita S and Fanton, Alison and Moon, Hyungseok C and Hie, Brian L and Konermann, Silvana and Hsu, Patrick D},
  journal={Science},
  pages={1820},
  year={2026},
  publisher={American Association for the Advancement of Science}
}

@article{roohani2024predicting,
  title={Predicting transcriptional outcomes of novel multigene perturbations with {GEARS}},
  author={Roohani, Yusuf and Huang, Kexin and Leskovec, Jure},
  journal={Nature Biotechnology},
  volume={42},
  number={6},
  pages={927--935},
  year={2024},
  publisher={Nature Publishing Group US New York}
}

@article{basu2020influence,
  title={Influence functions in deep learning are fragile},
  author={Basu, Samyadeep and Pope, Philip and Feizi, Soheil},
  journal={International Conference on Learning Representations},
  year={2021}
}

@article{chen2020simple,
    author = {Chen, Ting and Kornblith, Simon and Norouzi, Mohammad and Hinton, Geoffrey},
    year = {2020},
    title = {A simple framework for contrastive learning of visual representations},
    journal = {International Conference on Machine Learning},
}

@article{ren2021survey,
  title={A survey of deep active learning},
  author={Ren, Pengzhen and Xiao, Yun and Chang, Xiaojun and Huang, Po-Yao and Li, Zhihui and Gupta, Brij B and Chen, Xiaojiang and Wang, Xin},
  journal={ACM Computing Surveys},
  volume={54},
  number={9},
  pages={1--40},
  year={2021},
  publisher={ACM New York, NY}
}

@article{szklarczyk2025string,
  title={The {STRING} database in 2025: protein networks with directionality of regulation},
  author={Szklarczyk, Damian and Nastou, Katerina and Koutrouli, Mikaela and Kirsch, Rebecca and Mehryary, Farrokh and Hachilif, Radja and Hu, Dewei and Peluso, Matteo E and Huang, Qingyao and Fang, Tao and others},
  journal={Nucleic Acids Research},
  volume={53},
  number={D1},
  pages={D730--D737},
  year={2025},
  publisher={Oxford University Press}
}

@article{aleksander2025gene,
  title={The {Gene Ontology} knowledge-base in 2026},
  author={Aleksander, Suzi A and Balhoff, James P and Carbon, Seth and Cherry, J Michael and Ebert, Dustin and Feuermann, Marc and Gaudet, Pascale and Harris, Nomi L and Hill, David P and Kalita, Patrick and others},
  journal={Nucleic Acids Research},
  volume={54},
  number={D1},
  pages={gkaf1292},
  year={2025}
}

@article{weinberger2023isolating,
  title={Isolating salient variations of interest in single-cell data with {contrastiveVI}},
  author={Weinberger, Ethan and Lin, Chris and Lee, Su-In},
  journal={Nature Methods},
  volume={20},
  number={9},
  pages={1336--1345},
  year={2023},
  publisher={Nature Publishing Group US New York}
}

@article{kudo2026scalable,
  title={Scalable multimodal mapping of macrophage regulatory architecture by integrating optical and transcriptomic pooled screens},
  author={Kudo, Takamasa and Lopez, Romain and Meireles, Ana M and Rios, Antonio and Coelho, Paula and Chandrasekar, Vineethkrishna and Lam, Dylan C and Huetter, Jan-Christian and Ota, Mineto and Rozenblatt-Rosen, Orit and others},
  journal={bioRxiv},
  pages={2026--05},
  year={2026},
  publisher={Cold Spring Harbor Laboratory}
}
\bibliographystyle{iclr2027_conference}

\newpage

\subsection*{Acknowledgements and Disclosure of Funding}
We thank Takamasa Kudo for discussions about the myeloid cell screening dataset. This work was supported through the NYU IT High Performance Computing (HPC) resources, services, and staff expertise. Rushil Gupta was partially supported by a Google TPU Research \& Education Award. 

Disclosures: Romain Lopez has been an employee of Genentech since August 31st, 2021.

\subsection*{Code Availability Statement}
The code to reproduce the figures and tables, as well as to run the model can be found at \href{https://github.com/biologicalML/FractAL}{https://github.com/biologicalML/FractAL}.

\subsection*{Data Availability Statement}
All the classification datasets (CIFAR-10)\footnote{Dataset Source: \url{https://www.cs.toronto.edu/~kriz/cifar.html}}, standard regression datasets (KEGG, SARCOS and Diamonds)\footnote{KEGG: \url{https://archive.ics.uci.edu/ml/datasets/KEGG+Metabolic+Reaction+Network+\%28Undirected\%29}}\footnote{SARCOS Source: \url{http://www.gaussianprocess.org/gpml/data/}}\footnote{Diamonds Source: \url{https://www.openml.org/d/42225}} and the CD4+ T cell genome-wide Perturb-seq data\footnote{Dataset Source: \url{https://virtualcellmodels.cziscience.com/dataset/genome-scale-tcell-perturb-seq}} are publicly accessible. The OPS images and the corresponding Perturb-Seq data for the BMDMs experiment~\citep{kudo2026scalable} are currently undergoing publication as a
separate manuscript. They will become publicly available upon acceptance.

%%%%%%%%%%%%%%%%%%%%%%%%%%%%%%%%%%%%%%%%%%%%%%%%%%%%%%%%%%%%
% \newpage
\appendix
\section*{Appendices}
The following list gives an overview of the structure of the Appendix:
\begin{itemize}
\item Appendix~\ref{app:proof}: derivation of the strategy-wise influence score.
\item Appendix~\ref{app:fractal_details}: practical estimation of influence scores and construction of the composite batch.
\item Appendix~\ref{app:strat-pool}: the acquisition strategies in our portfolio.
\item Appendix~\ref{app:baselines}: details and pseudocode for each strategy selection baseline.
\item Appendix~\ref{app:precision_k}: definition of the allocation diagnostics.
\item Appendix~\ref{app:exp-setups}: datasets, surrogate models, and hyperparameters for each setup.
\item Appendix~\ref{app:perf-setups}: aggregate gain and rank of each method across the seven setups.
\item Appendix~\ref{app:complete-tables}: per-setup results with the pure strategies added.
\item Appendix~\ref{app-ablation}: ablations on the smoothing parameter, allocation update, reward attribution, and soft allocation.
\item Appendix~\ref{app:runtime}: wall-clock runtime by phase for each method.
\item Appendix~\ref{app:more-plots}: budget allocation dynamics for all methods and setups.
\end{itemize}

\clearpage

\section{Proof of Proposition~\ref{lem:strategy-influence}}
\label{app:proof}
\strategyinfluence*

\paragraph{Assumptions.} These assumptions are standard in the influence-function literature~\citep{koh2017understanding}. They do not hold exactly for deep networks trained by stochastic gradient descent using early-stopping. We describe those, as well as practical workarounds in Appendix~\ref{app:fractal_details}.

\begin{proof}
We adapt the standard influence-function derivation~\citep{koh2017understanding, xu2019understanding} to our portfolio setting, where the perturbation vector $\boldsymbol{\epsilon} \in \mathbb{R}^K$ has one component per strategy rather than one per training point. The proof proceeds in two steps. 
First, we derive the parameter sensitivity $\frac{d\hat{\theta}^t_{\boldsymbol{\epsilon}}}{d\boldsymbol{\epsilon}}$ at $\boldsymbol{\epsilon} = \mathbf{0}$ via the implicit function theorem. 
Second, we apply the chain rule to obtain the closed form for $\mathcal{I}_k(\hat{\theta}^t)$.

\paragraph{Step 1: Parameter Sensitivity}
We first recall the perturbed risk from Eq.~\ref{eq:perturb_risk}: for a perturbation vector $\boldsymbol{\epsilon} \in \mathbb{R}^K$,
\begin{equation}
    \mathcal{F}_t(\boldsymbol{\epsilon}, \theta) = \mathcal{R}(\mathcal{D}^t_\text{train}; \theta) + \sum_{k=1}^K\frac{\epsilon_k}{n^t_k} \sum_{(x, y) \in \mathcal{B}^t_k} \mathcal{L}\left((x, y); \theta\right).
\end{equation}
The perturbed parameter $\hat{\theta}^t_{\boldsymbol{\epsilon}}$ is its minimizer:
\begin{equation}
    \hat{\theta}^t_{\boldsymbol{\epsilon}} := \arg\min_\theta \mathcal{F}_t(\boldsymbol{\epsilon}, \theta),
\end{equation}
with $\hat{\theta}^t_{\mathbf{0}} = \hat{\theta}^t$. 

By optimality, it satisfies the first-order condition:
\begin{equation}
\label{eq:opt-cond}
    \nabla_\theta \mathcal{R}\left(\mathcal{D}^t_\text{train}; \hat{\theta}^t_{\boldsymbol{\epsilon}}\right) + \sum_{k=1}^K \frac{\epsilon_k}{n^t_k} \sum_{(x, y) \in \mathcal{B}^t_k} \nabla_\theta \mathcal{L}\left((x, y); \hat{\theta}^t_{\boldsymbol{\epsilon}}\right) = 0.
\end{equation}

Let $\Delta_{\boldsymbol{\epsilon}} := \hat{\theta}^t_{\boldsymbol{\epsilon}} - \hat{\theta}^t$. As $\boldsymbol{\epsilon} \to \mathbf{0}$, we have $\Delta_{\boldsymbol{\epsilon}} \to \mathbf{0}$.

Since $\mathcal{L}$ is twice continuously differentiable and
$H_{\hat{\theta}^t}$ is positive definite, the implicit function theorem
applied to Eq.~\ref{eq:opt-cond} implies that, for $\boldsymbol{\epsilon}$ near $\mathbf{0}$,
the minimizer $\hat{\theta}^t_{\boldsymbol{\epsilon}}$ near $\hat{\theta}^t$
is a continuously differentiable function of $\boldsymbol{\epsilon}$. In
particular, $\Delta_{\boldsymbol{\epsilon}} = O(\|\boldsymbol{\epsilon}\|)$
as $\boldsymbol{\epsilon} \to \mathbf{0}$. Taylor-expanding Eq.~\ref{eq:opt-cond} around $\hat{\theta}^t$ and discarding terms of order $\boldsymbol{\epsilon} \Delta_{\boldsymbol{\epsilon}}$ and higher:
\begin{align}
    0 = \, & \nabla_\theta \mathcal{R}(\mathcal{D}^t_\text{train}; \hat{\theta}^t) + \nabla^2_\theta \mathcal{R}(\mathcal{D}^t_\text{train}; \hat{\theta}^t) \, \Delta_{\boldsymbol{\epsilon}} + \sum_{k=1}^K \frac{\epsilon_k}{n^t_k} \sum_{(x, y) \in \mathcal{B}^t_k} \nabla_\theta \mathcal{L}\left((x, y); \hat{\theta}^t\right) + O(\|\boldsymbol{\epsilon}\|^2).
\end{align}
Since $\hat{\theta}^t$ minimizes $\mathcal{R}(\mathcal{D}^t_\text{train}; \cdot)$, the first term vanishes: $\nabla_\theta \mathcal{R}(\mathcal{D}^t_\text{train}; \hat{\theta}^t) = 0$. Recognizing $\nabla^2_\theta \mathcal{R}(\mathcal{D}^t_\text{train}; \hat{\theta}^t) = H_{\hat{\theta}^t}$ and using the regularity assumption to invert it, we solve for $\Delta_{\boldsymbol{\epsilon}}$:
\begin{equation}
\label{eq:delta-eps}
    \Delta_{\boldsymbol{\epsilon}} = -H_{\hat{\theta}^t}^{-1} \left[\sum_{k=1}^K \frac{\epsilon_k}{n^t_k} \sum_{(x, y) \in \mathcal{B}^t_k} \nabla_\theta \mathcal{L}\left((x, y); \hat{\theta}^t\right) 
    + O(\|\boldsymbol{\epsilon}\|^2)\right]
\end{equation}
This expresses the parameter shift $\Delta_{\boldsymbol{\epsilon}}$ as a function of $\boldsymbol{\epsilon}$, with each $\epsilon_k$ scaling its strategy's gradient contribution.

\paragraph{Step 2: Chain Rule for the Influence Vector}

By definition, the $k$-th component of the influence vector is the partial derivative of the validation risk at the perturbed parameter, evaluated at $\boldsymbol{\epsilon} = \mathbf{0}$:
\begin{equation}
    \mathcal{I}_k(\hat{\theta}^t) = \frac{\partial}{\partial \epsilon_k} \mathcal{R}\left(\mathcal{D}_\text{val}; \hat{\theta}^t_{\boldsymbol{\epsilon}}\right) \bigg|_{\boldsymbol{\epsilon} = \mathbf{0}}.
\end{equation}
Applying the chain rule:
\begin{equation}
    \mathcal{I}_k(\hat{\theta}^t) = \nabla_\theta \mathcal{R}(\mathcal{D}_\text{val}; \hat{\theta}^t)^\top \frac{\partial \hat{\theta}^t_{\boldsymbol{\epsilon}}}{\partial \epsilon_k}\bigg|_{\boldsymbol{\epsilon} = \mathbf{0}}.
\end{equation}
The Jacobian column on the right is read off from Eq.~\ref{eq:delta-eps} by differentiating with respect to $\epsilon_k$ at $\boldsymbol{\epsilon} = \mathbf{0}$ (higher order terms of $\boldsymbol{\epsilon}$ vanish here at $\boldsymbol{\epsilon}=\mathbf{0}$):
\begin{equation}
    \frac{\partial \hat{\theta}^t_{\boldsymbol{\epsilon}}}{\partial \epsilon_k}\bigg|_{\boldsymbol{\epsilon} = \mathbf{0}} = -H_{\hat{\theta}^t}^{-1} \left( \frac{1}{n^t_k} \sum_{(x, y) \in \mathcal{B}^t_k} \nabla_\theta \mathcal{L}\left((x, y); \hat{\theta}^t\right) \right).
\end{equation}
Substituting into the chain rule expression:
\begin{equation}
    \mathcal{I}_k(\hat{\theta}^t) = -\nabla_\theta \mathcal{R}(\mathcal{D}_\text{val}; \hat{\theta}^t)^\top H_{\hat{\theta}^t}^{-1} \left( \frac{1}{n^t_k} \sum_{(x, y) \in \mathcal{B}^t_k} \nabla_\theta \mathcal{L}\left((x, y); \hat{\theta}^t\right) \right),
\end{equation}
which matches Eq.~\ref{eq:infl-cal-our}.
\end{proof}

\clearpage
\section{Implementation Details for FractAL}
\label{app:fractal_details}
\subsection{Practical Estimation of Influence Scores}
\label{app-prac-est}

Proposition~\ref{lem:strategy-influence} assumes an exact empirical risk minimizer and a positive-definite Hessian, and is stated for the full parameter vector. None of these holds exactly in our setting. We describe below how each departs from the idealized statement, the approximation we use in response, and empirical evidence that the resulting estimator remains reliable.

\paragraph{Optimality of the trained parameters.} The derivation uses the first-order optimality condition $\nabla_\theta R(\mathcal{D}^t_{\text{train}}; \hat\theta^t) = 0$ to eliminate the leading term of the Taylor expansion. In practice however, $\hat\theta^t$ is obtained by solving a stochastic optimization problem with a first-order method, using an early stopping criterion on the validation set. Therefore, this condition holds only approximately.

\paragraph{Positive definiteness of the Hessian.} The Hessian of the training risk at the trained parameters, $H_{\hat\theta^t} = \nabla^2_\theta R(\mathcal{D}^t_{\text{train}}; \hat\theta^t)$, need not be positive definite. For the multiclass softmax head used on CIFAR-10, it is provably singular since a common shift of all logits leaves the loss unchanged. We therefore add an explicit damping term before inversion in every setup, so
that the system we invert is $H_{\hat\theta^t} + \lambda_{\text{damp}} I$ with $\lambda_{\text{damp}} = 0.01$.

\paragraph{Last-layer restriction.} Exact computation of the influence scores may become intractable for surrogate models with a large number of parameters. We therefore restrict the Hessian to the last-layer parameters $\theta_L$~\citep{koh2017understanding}. In the case of ResNet-18, this corresponds to the final linear classification layer, and for the MLPs it is the final fully connected layer.

\paragraph{Reliability of the approximation.} We validate the resulting scores by comparing them against the retraining-based signal they approximate. Across rounds on CIFAR-10 (C1), the two rewards rank strategies consistently, with a Spearman correlation of $0.7$ - $0.8$ (Table~\ref{infl-retr-corr-results}). Additionally, substituting one for the other leaves downstream performance comparable on all seven setups (Table~\ref{infl-retr-results}). 

\subsection{Forming the Composite Batch}
\label{app:meth-detail}

At each round $t$, the composite batch $\mathcal{B}^t$ is formed by combining sub-batches independently proposed by each strategy in the portfolio. We describe here the procedure used to guarantee that the realized batch contains exactly $B$ unique points with each attributed to a single strategy (Algorithm~\ref{alg:FractAL}).

\paragraph{Budget rounding.} Since $Bp^t_k$ is generally not an integer, we allocate $n^t_k = \lfloor B p^t_k \rfloor$ points to each strategy $\mathcal{A}_k$ and assign the remainder $B - \sum_k \lfloor B p^t_k \rfloor$ to the strategy with the largest allocation, $\mathcal{A}_j$ with $j = \arg\max_k p^t_k$ (ties broken by strategy index). The realized sub-batch sizes therefore sum to exactly $B$.

\paragraph{Overlap resolution.} All strategies select candidates exclusively from the currently unlabeled pool $\mathcal{D}_{\text{pool}} \setminus \mathcal{D}^{t-1}_{\text{train}}$, so no previously acquired point can be re-selected. Within a round, however, two strategies may independently propose the same unlabeled point. We resolve this by having strategies select sequentially, in decreasing order of their allocated budget $p^t_k$ (ties broken by strategy index): each strategy's candidate set excludes points already claimed by a higher-allocated strategy earlier in the same round. This guarantees the realized sub-batches $\mathcal{B}^t_1, \ldots, \mathcal{B}^t_K$ are disjoint, and that every acquired point is attributed to exactly one strategy in the reward computation.

\paragraph{Empirical overlap rate.} We measured how often this overlap resolution rule is actually invoked by logging, prior to resolution, how frequently two strategies' independently-ranked candidate sets would have proposed the same point. On CIFAR-10 (C2, $K{=}8$, $B{=}500$), an average of 3 of the 500 points acquired per round were claimed by more than one strategy (0.6\% of the batch). The rule therefore affects well under 1\% of any batch, and the choice of tie-breaking order therefore has a negligible effect on the reported results.

\clearpage
\section{Active Learning Strategies}
\label{app:strat-pool}
Below we provide details about each of the data acquisition strategies used in our pool. The first four are suitable for all experiments:
\begin{itemize}
    \item \textbf{Random} is a strategy that randomly chooses a given number of points from the unlabeled pool
    \item \textbf{ProbCover} \citep{yehuda2022active} selects unlabeled points that maximize coverage of the data distribution under a probabilistic similarity graph, ensuring chosen samples ``cover'' as many other points as possible.
    \item \textbf{TypiClust} \citep{hacohen2022active} performs $k$-means clustering in an embedding space and samples typical points from the clusters ordered by their respective total size and labeled set size.
    \item \textbf{CoreSet} \citep{sener2017active} iteratively picks the most distant points from the already labeled points, given an embedding of all the unlabeled inputs. We use the Euclidean distance as the distance metric in our experiments.
\end{itemize}
The last four apply specifically to classification tasks in our work:
\begin{itemize}
    \item \textbf{BADGE.} BADGE \citep{ash2019deep} represents each unlabeled sample as the gradient of the cross-entropy loss it induces on the last layer of the surrogate model and then samples diverse points in the gradient embedding space.
    \item \textbf{LeastConf} strategy picks the samples that have the lowest maximum predicted probability.
    \item \textbf{MaxEnt} picks samples with highest entropy in the predicted probabilities.
    \item \textbf{MinMargin} picks the samples which have minimum difference between the probability of top two predicted classes.
\end{itemize}

\clearpage
\section{Strategy Selection Baselines}
\label{app:baselines}
\subsection{ALBL: Active Learning by Learning}

ALBL \citep{hsu2015active} utilizes the Exp4.P bandit solver to treat the strategies as experts giving advice over the unlabeled data points. At iteration $t$, the solver maintains weights $w^t_k$ and determines the selection probability of experts $s^t_k$ as:
\begin{equation}
    s^t_k = (1 - K p_{\min}) \frac{w^t_k}{\sum_{k'=1}^K w^t_{k'}} + p_{\min}
\end{equation}
where we set $p_{\min} = 1/B$. The strategy for the next iteration is then sampled with the distribution $s^t$. Let's assume the chosen strategy is $\mathcal{A}_k$. Thus, the allocation vector $\mathbf{p}^t \in \mathbb{R}^K$ is given by
\begin{equation}
p^t_{k'} =
\begin{cases}
1 & \text{if } k' = k, \\
0 & \text{otherwise},
\end{cases}
\qquad \forall k' \in [K],
\end{equation}
Conceptually, ALBL acquires the whole batch of points using $\mathcal{A}_k$ and receives a partial information reward $\mathbf{r}^t \in [0,1]^K$ based on validation performance, defined as follows:
\begin{equation}
\label{eq:albl-rew}
    r^t_k = \begin{cases}
        \text{Accuracy on } \mathcal{D}_{\text{val}} & \text{Classification Task}\\
        \frac{1}{1 + \text{Loss}(\mathcal{D}_{\text{val}})} & \text{Regression Task}\\
    \end{cases}
\end{equation}
and $r^t_{k'} = 0$ for all $k'\neq k$. This reward is then utilized to update strategy selection weights using an exponential rule:
\begin{equation}
    w^{t+1}_k = w^t_k \exp \left( \frac{p_{\min}}{2} (\hat{y}_k(t) + \hat{v}_k(t) p_{\min}) \right)
\end{equation}
where $\hat{y}_k(t)$ and $\hat{v}_k(t)$ are the reward and variance estimators. The whole algorithm is explained in Algorithm \ref{alg:albl-exp4p}.

\begin{algorithm}[h]
\caption{ALBL: Active Learning by Learning}
\label{alg:albl-exp4p}
\begin{algorithmic}[1]
\Require Strategy portfolio $\mathcal{P} = \{\mathcal{A}_1, \dots, \mathcal{A}_K\}$; pool $\mathcal{D}_\text{pool}$; initial labeled set $\mathcal{D}_\text{init}$; validation set $\mathcal{D}_\text{val}$; batch size $B$; rounds $T$
\State \textbf{Initialization:} $w_k^1 = 1, \forall k \in [K]$; $\mathcal{D}^0_{\text{train}} = \mathcal{D}_{\text{init}}$
\For{$t = 1, 2, \ldots T-1$}
    \State Receive advice vectors $\mathbf{\xi}^1(t), \dots, \mathbf{\xi}^K(t)$, with $\mathbf{\xi}^k(t) = \mathbf{e}_k \in \{0,1\}^K$
    \State $W^t = \sum_{k=1}^K w_k^t$
    \For{$k = 1, \dots, K$}
        \State $s_k^t = (1 - K p_{\min}) \sum_{k'=1}^K \frac{w_{k'}^t\xi_k^{k'}(t)}{W^t} + p_{\min} = (1 - K p_{\min}) \frac{w_k^t}{W^t} + p_{\min}$
    \EndFor
    \State Draw action $k^t \sim (s_1^t, \dots, s_K^t)$
    \State Define allocation vector $\mathbf{p}^t$ with $p^t_{k^t} = 1$ and $p^t_{k'} = 0, \forall k' \in [K], k' \neq k^t$
    \State Sample batch $\mathcal{B}^t$:
    \[
    \mathcal{B}^t \leftarrow \text{Sample}(\mathcal{A}_{k^t}, B, \mathcal{D}_{\text{pool}}\setminus \mathcal{D}_{\text{train}}^{t-1})
    \]
    \State Query labels for $\mathcal{B}^t$ and update the training dataset:
    \[
    \mathcal{D}^t_{\text{train}} \leftarrow \mathcal{D}^{t-1}_{\text{train}} \cup \{(x, y): x \in \mathcal{B}^t\}
    \]
    \State Retrain model on $\mathcal{D}^t_{\text{train}}$
    \State Receive reward $\mathbf{r}^t \in [0,1]^K$ as in Equation \ref{eq:albl-rew}
    \State Set $\hat{r}^t_{k^t} \gets \frac{r^t_{k^t}}{s^t_{k^t}}$ and $\hat{r}^t_{k'} = 0, \forall k' \in [K], k' \neq k^t$
    \For{$k = 1, \dots, K$}
        \State $\hat{y}_k^t = \mathbf{\xi}^k(t) \cdot \mathbf{\hat{r}^t}$
        \State $\hat{v}_k^t = \nicefrac{1}{s_{k}^t}$
        \State $w_k^{t+1} = w_k^t \exp\left( \frac{p_{\min}}{2} \left( \hat{y}_k^t + \hat{v}_k^t * p_{\min}\right) \right)$
    \EndFor
\EndFor
\end{algorithmic}
\end{algorithm}

\clearpage

\subsection{Upper Confidence Bound (UCB)}
\label{sec:UCB}
UCB is a classic multi-armed bandit algorithm that we apply for active learning. Each arm is a strategy $\mathcal{A}_k$ from the portfolio. At any iteration $t$, based on the confidence bound calculated from the mean reward and associated uncertainty, an arm $k^t$ is decided to be played. Assume $\tau$ is the performance metric that we use to evaluate the quality of model parameters (typically negative validation loss for regression tasks and validation accuracy for classification tasks). Therefore, the partial information reward $\mathbf{r}^t \in \mathbb{R}^K$ obtained as consequence of the action in iteration $t$ is defined as:
\begin{equation}
\label{eq:ucb-rew}
r^t_k =
\begin{cases}
\tau(\hat{\theta}^t) - \tau(\hat{\theta}^{t-1}) & \text{if } k = k^t, \\
0 & \text{otherwise,}
\end{cases}
\qquad \forall k \in [K],
\end{equation}
where $\hat{\theta}^t$ denotes the parameters of the surrogate model obtained
after training on $\mathcal{D}^t_{\mathrm{train}}$ (see Algorithm~\ref{alg:ucb_batch_al}).
\begin{algorithm}[h]
\caption{UCB: Upper Confidence Bound}
\label{alg:ucb_batch_al}
\begin{algorithmic}[1]
\Require Strategy portfolio $\mathcal{P} = \{\mathcal{A}_1, \dots, \mathcal{A}_K\}$; pool $\mathcal{D}_\text{pool}$; initial labeled set $\mathcal{D}_\text{init}$; validation set $\mathcal{D}_\text{val}$; batch size $B$; rounds $T$, exploration parameter $c$, pseudo-count $\gamma$, model performance metric $\tau$ 
\State \textbf{Initialization:} $D^0_{\text{train}} = \mathcal{D}_{\text{init}}; c \gets 1$
\State \textbf{Initialization:} counts $n_k \leftarrow \gamma$, values $\hat{\mu}_k \leftarrow 0$ for all $k \in [K]$

\State Train model on $\mathcal{D}^{0}_{\text{train}}$ and obtain $\hat{\theta}^{0}$

\For{ $t = 1, 2,\ldots, T-1$}
    \State Select strategy:
    \[
    k^t = \arg\max_k \left( \hat{\mu}_k + c \sqrt{\frac{\log t}{n_k}} \right)
    \]
    \State Define allocation vector $\mathbf{p}^t$ with $p^t_{k^t} = 1$ and $p^t_{k'} = 0, \forall k' \neq k^t$
    \State Sample batch $\mathcal{B}^t$:
    \[
    \mathcal{B}^t \leftarrow \text{Sample}(\mathcal{A}_{k^t}, B, \mathcal{D}_{\text{pool}}\setminus \mathcal{D}_{\text{train}}^{t-1})
    \]
    \State Query labels for $\mathcal{B}^t$ and update the training dataset:
    \[
    \mathcal{D}^t_{\text{train}} \leftarrow \mathcal{D}^{t-1}_{\text{train}} \cup \{(x, y): x \in \mathcal{B}^t\}
    \]
    \State Retrain model on $\mathcal{D}^t_{\text{train}}$ and obtain optimal parameters $\hat{\theta}^t$
    \State Receive reward $\mathbf{r}^t \in \mathbb{R}^K$ as in Eq. \ref{eq:ucb-rew}
    \State Update UCB statistics. For all $k \in [K]$:
    \[
    n_{k} \leftarrow n_{k} + \mathds{1}_{k=k^t}
    \]
    \[
    \hat{\mu}_{k} \leftarrow \hat{\mu}_{k} + \mathds{1}_{k = k^t} \left[\frac{1}{n_{k}} (r^t_k - \hat{\mu}_{k})\right]
    \]
\EndFor
\end{algorithmic}
\end{algorithm}

\subsection{SCRiBLe Algorithm}
SCRiBLe \citep{abernethy2009competing} is a linear bandit solver we apply for active learning in this work. Specifically, SCRiBLe operates directly on the allocation simplex $\Delta^{K-1}$. At each round it proposes an allocation vector and observes the resulting validation loss, which it treats as the cost to be minimized. 
Consequently, the action space for this algorithm is the convex compact simplex $\Delta^{K-1}$. For the (K-1)-dimensional simplex $\Delta^{K-1} = \{x \in \mathbb{R}^K: x \geq 0, \sum\limits_k x_k = 1\}$, the log-barrier is:
\begin{equation}
    \mathcal{C}(x) = -\sum\limits_{k=1}^{K-1} \log x_k - \log (1 - \sum\limits_{k=1}^{K-1} x_k)
\end{equation}
where $x_K = 1 - \sum\limits_{k=1}^{K-1} x_k$. The algorithm starts close to the center of the simplex, i.e. a uniform distribution and explores the domain while still being inside the simplex. At any iteration $t$, the algorithm perturbs the proportion vector in a randomly chosen direction of eigenvector of $\nabla^2\mathcal{C}(\mathbf{p}^t)$ and obtains a batch according to this updated proportion. The whole algorithm is described in Algorithm \ref{alg:scribble}.

\begin{algorithm}[h]
\caption{SCRiBLe}
\label{alg:scribble}
\begin{algorithmic}[1]
\Require Strategy portfolio $\mathcal{P} = \{\mathcal{A}_1, \dots, \mathcal{A}_K\}$; pool $\mathcal{D}_\text{pool}$; initial labeled set $\mathcal{D}_\text{init}$; validation set $\mathcal{D}_\text{val}$; batch size $B$; rounds $T$, $\eta > 0$, $\vartheta$-self-concordant barrier $\mathcal{C}$, $\mathcal{K} = \Delta^{K-1}$
\State \textbf{Initialization:} $D^0_{\text{train}} = \mathcal{D}_{\text{init}}$; $\mathbf{p}^1 \gets \arg\min_{\mathbf{p} \in \mathcal{K}} \mathcal{C}(\mathbf{p})$ = $\{\frac{1}{K}, \cdots, \frac{1}{K}\}$
\For{$t = 1, \dots, T-1$}
    \State Compute eigenpairs $\{(\mathbf{e}_k, \lambda_k)\}_{k=1}^{K-1}$ of $\nabla^2 \mathcal{C}(\mathbf{p}^{t})$
    \State Sample $k^t \sim \mathrm{Unif}(\{1,\dots,K-1\})$, $\;\varepsilon^t \sim \{-1,+1\}$ uniformly
    \State $\mathbf{p}^t \gets \mathbf{p}^t + \varepsilon^t \lambda_{k^t}^{-1/2} \mathbf{e}_{k^t}$
    \State Sample strategy sub-batches: $\mathcal{B}^t_i \leftarrow \text{Sample}(\mathcal{A}_{i}, \lfloor p^t_i * B\rfloor, \mathcal{D}_{\text{pool}}\setminus \mathcal{D}^{t-1}_{\text{train}})$
    \State Aggregate batch $\mathcal{B}^t = \bigcup\limits_{k=1}^K \mathcal{B}^t_k; $
    \State Query labels for $\mathcal{B}^t$ and update the training dataset:
    \[
    \mathcal{D}^t_{\text{train}} \leftarrow \mathcal{D}^{t-1}_{\text{train}} \cup \{(x, y): x \in \mathcal{B}^t\}
    \]
    \State Retrain model on $\mathcal{D}^t_{\text{train}}$ and obtain optimal parameters $\hat{\theta}^t$
    \State Compute cost $c^t \defeq \mathcal{R}(\mathcal{D}_{\text{val}}; \hat{\theta}^t)$
    \State $\tilde{\mathbf{c}}^t \gets (K-1) \, c^t \, \varepsilon^t \lambda_{k^t}^{1/2} \mathbf{e}_{k^t}$
    \State $\mathbf{p}^{t+1} \gets \arg\min_{\mathbf{p} \in \mathcal{K}} \left( \eta \sum_{s=1}^{t} \tilde{\mathbf{c}}^{s\top} \mathbf{p} + \mathcal{C}(\mathbf{p}) \right)$
\EndFor
\end{algorithmic}
\end{algorithm}
\subsection{SelectAL}
\label{app:selectal}
SelectAL~\citep{hacohen2024select} was originally introduced as a retraining-based active learning framework that treats strategy selection as a choice between two regimes: a ``low-budget'' regime focused on representativeness and a ``high-budget'' regime focused on uncertainty and atypicality. To determine the optimal regime at each iteration, SelectAL assesses the utility of a strategy as follows. At any iteration $t$, it treats the labeled set $\mathcal{D}^t_{\text{train}}$ as if it were unlabeled and removes a small subset of the points, specifically $\max(\lfloor 0.05 \cdot |\mathcal{D}^t_{\text{train}}| \rfloor, 1)$ (the $5\%$ value follows the original publication), chosen according to a candidate strategy's logic, and measures the resulting change in model performance. The strategy whose deletion induces the largest drop in performance is judged most informative and is selected for the subsequent forward acquisition step.

While the original work suggests using one fixed representative strategy for deletion per budget category (which may differ from the corresponding category's acquisition strategy), such a constraint limits the framework's applicability in complex portfolios like our setup (C2) (Section~\ref{clas:intro}). Therefore, in our implementation, we allow SelectAL to evaluate all $K$ strategies in the portfolio by setting the forward-selection and deletion-based strategies to be identical for each candidate. This adaptation ensures a fair comparison by allowing SelectAL to operate over the same strategy pool as FractAL.

\subsection{AutoAL}
AutoAL~\citep{wang2024autoal} frames strategy selection as a differentiable search problem via two co-trained networks. First, FitNet estimates the informativeness of each unlabeled sample. Second, SearchNet learns a continuous weighting over the strategy portfolio. Each candidate strategy contributes a score to every unlabeled point, and SearchNet combines these into a single per-point score via gradient-based optimization on the labeled set (relaxing what would otherwise be a discrete strategy search into a continuous one, in the spirit of differentiable architecture search); the points with the highest combined scores are then selected for labeling. Unlike FractAL, which allocates budget fractionally across strategies to form disjoint sub-batches, AutoAL blends strategies at the level of individual point scores, and its bilevel optimization introduces additional hyperparameters and per-round computational overhead from jointly training the two auxiliary networks.

\clearpage
\section{Allocation Diagnostics}
\label{app:precision_k}

We seek to establish a quantitative notion of quality for a strategy portfolio allocation that can be compared across multiple adaptive strategy selection methods.  One potential avenue for this diagnostic is to consider ``pure-strategy'' test performance (i.e., running each strategy alone across the full AL run), to decide on which strategies are empirically best- and worst-performing at each round. Based on this evaluation, a natural desideratum for a soft strategy selection method is to concentrate budget on the strategies that perform well and withhold it from those that do not.

This criterion is not valid in general, because samples from different strategies may be complementary. In those cases, combining samples from a seemingly poor strategy with a strong one could create synergies and boost downstream performance. Our experiments, however, suggest this does not happen here. No strategy selection method in our study outperforms the best pure strategy in the portfolio (Tables~\ref{tab:clas-all-pure}, \ref{tab:reg-all-pure}, \ref{tab:bio-all-pure}). Remarkably, this holds even for AutoAL, which blends strategies at the level of individual point scores rather than allocating budget across them, and can therefore in principle select points that exploit complementarity between strategies. It nonetheless underperforms both random and a uniform allocation on every aggregate measure (Table~\ref{rand-impr-results}). We therefore take the pure-strategy ranking as our reference, while noting that portfolios with genuinely complementary strategies remain untested here.

On some tasks, several strategies perform comparably well, so identifying a single best (or worst) strategy is too strict a standard. A method that correctly favors any strategy within this comparable set should not be penalized for not isolating the single top performer. We therefore evaluate against the best and worst \emph{sets} of pure strategies, of cardinality $k$ (reported per task in Table~\ref{tab:recall-k}). At each AL round, Best-Set Overlap is the fraction of the best set that is also among the top-$k$ strategies by the method's internal allocation, and Worst-Set Overlap is the fraction of the worst set that is also among the bottom-$k$ strategies.

To determine $k$ for each task, we take the mean $\pm$ one standard deviation at each round of each pure strategy's test performance (across seeds, after averaging within splits) and count how many strategies' bands overlap the best-performing (or, for Worst-Set Overlap, the worst-performing) strategy's band; the cardinality of the set is one plus this count. We take the mode of this per-round value across rounds $t{=}2$ to $T{-}1$ as the fixed set size reported for that task. This is a practical heuristic rather than a formal hypothesis test, and portfolios whose strategies are more tightly clustered in performance naturally have larger cardinality sets.

\begin{table}[h]
\caption{Cardinality of the Best and Worst Sets for the allocation quality assessment for all the datasets}
    \label{tab:recall-k}
    \centering
    \small
    \begin{tabular}{l cc}
        \toprule
        Dataset & $|\text{Best Set}|$ & $|\text{Worst Set}|$\\
        \midrule
         CIFAR 10 (C1) & 1 & 1\\
         CIFAR 10 (C2) & 2 & 1\\
         KEGG & 2 & 1\\
         SARCOS & 2 & 3\\
         Diamonds & 2 & 2\\
        BMDMs (B1)  & 1 & 2\\
        T cells (B2)  & 1 & 2\\
        \bottomrule
    \end{tabular}
\end{table}

\clearpage
\section{Experimental Setups}
\label{app:exp-setups}
We now discuss the setups used in each experiment in more detail. A summary of different active learning configurations is presented in Table \ref{tab:al-config}.

\begin{table}[h]
\caption{Active learning configurations for each setup: surrogate model, sizes of the initial labeled, validation, test and pool sets, and batch size $B$. All strategy selection methods and pure strategies share the same configuration within a setup.}
    \label{tab:al-config}
    \centering
    \small
    \begin{tabular}{l cccccc}
        \toprule
        Dataset & Surrogate Model & $|\mathcal{D}_{\text{init}}|$ & $|\mathcal{D}_{\text{val}}|$ & $|\mathcal{D}_{\text{test}}|$ & $|\mathcal{D}_{\text{pool}}|$ & $B$\\
        \midrule
         CIFAR 10 & ResNet-18 & 500 & 1024 & 10000 &  48976 & 500\\
         % CIFAR 100 & ResNet-18 & 2000 & 2048 & 10000 & 47952 & 2000\\
         KEGG & 3 layer MLP & 500 & 1024 & 12921 & 50663 & 500\\
         SARCOS & 3 layer MLP & 500 & 1024 & 8896 & 34564 & 500\\
         Diamonds & 3 layer MLP & 500 & 1024 & 10788 & 42128 & 500\\
        BMDMs  & 1 layer MLP & 35 & 40 & 103 &  371 & 35\\
        CD4+ T cells  & 1 layer MLP & 35 & 40 & 141 & 528  & 35\\
        \bottomrule
    \end{tabular}
\end{table}

\subsection{Classification Dataset}
\label{app:cf10-mot}

\paragraph{Data processing. } We follow the data processing recipe from \citet{hacohen2024select}. Starting from the raw CIFAR 10 images, we first apply random cropping of the same size $32 \times 32$ with a padding of 4, followed by random horizontal flips. Post these augmentations, the pixel values are scaled to $[0, 1]$ and are then standardized using the channel-wise mean and standard deviation of $[0.4914, 0.4822, 0.4465]$ and $[0.247 , 0.2435, 0.2616]$ respectively.

\paragraph{Surrogate Model. } We use a ResNet-18~\citep{he2016deep} as the surrogate model for CIFAR 10. The ResNet is trained for 50 epochs in each active learning iteration. We used an SGD optimizer, with 0.9 Nesterov momentum, 0.0003 weight decay, cosine learning rate scheduling with a base learning rate of 0.025, dropout with probability $0.5$ and batch size of 100 examples. We borrow all the architectural and hyperparameter details from \citet{hacohen2024select}. We select the best checkpoint based on validation accuracy. For ProbCover, we use $\delta = 0.6$ as outlined in \citet{yehuda2022active}.

\paragraph{Embeddings. } We used the $512$ dimensional $\text{L}2-$normalized SimCLR~\citep{chen2020simple} embeddings\footnote{Embeddings obtained from: \href{https://github.com/fbickfordsmith/ssl-embeddings}{https://github.com/fbickfordsmith/ssl-embeddings}}  for strategies like CoreSet, ProbCover and TypiClust, that operate in the embedding space. Note that the embeddings remain fixed throughout the AL run, as prescribed in \citet{hacohen2024select}. 

\subsection{Regression Datasets}
\label{app:reg-setup}

\paragraph{Data processing. } All three regression datasets, KEGG, SARCOS and Diamonds use the same data processing pipeline. We closely follow \citet{holzmuller2023framework} for the data processing recipe. Specifically, we process the input features as follows:
\begin{equation}
    x_j^{\text{processed}} = 5 \text{tanh}\left(\frac{1}{5}\frac{x_j - \hat{\mu}_j}{\hat{\sigma}_j}\right)
\end{equation}
where $\hat{\mu}$ and $\hat{\sigma}$ are the feature wise mean and standard deviation of the dataset $\mathcal{D}_{\text{pool}}$. This is done to make sure all input features lie in the range $[-5,5]$ and simplify the learning process for the MLP.

\paragraph{Surrogate Model. } We take inspiration from \citet{holzmuller2023framework} for the surrogate model and use a 3 layer MLP with ReLU activations and 2 hidden layers of 512 dimension each. We use a weight decay of 0.0003 for regularization and optimize the network using Adam optimizer with a base learning rate of 0.01. At every AL iteration, the surrogate model is trained for 256 epochs with a batch size of 64 and a learning rate schedule that decays linearly with the epochs. We select the best checkpoint based on validation loss.

\paragraph{Embeddings. } We use the processed input features concatenated together as the embedding for each input sample for strategies like CoreSet, ProbCover and TypiClust.

\paragraph{Portfolio.} For our portfolio, we choose only Random, CoreSet, ProbCover and TypiClust as they naturally apply to these setups. MaxEnt, MinMargin and LeastConf are more suited for classification tasks as they operate with logits and hence are removed from the portfolio for regression setups.

\subsection{High-Throughput Perturbation Screens}
\label{app:pert-pair}

\paragraph{B1 (BMDMs) data processing.} Starting from the matched Perturb-seq and OPS count data~\citep{kudo2026scalable}, we processed the transcriptomic readouts with \texttt{scanpy} for quality control and normalization, used ContrastiveVI~\citep{weinberger2023isolating} to isolate perturbation-induced variation from background biological variation, and computed per-perturbation log-fold changes (against non-targeting controls) with lvm-DE~\citep{boyeau2023empirical} across 4{,}000 highly variable genes. Imaging embeddings were obtained by averaging CellProfiler features~\citep{carpenter2006cellprofiler} per perturbation and retaining the top 70 principal components.

\paragraph{B2 (CD4+ T cells) data processing.} We used the publicly available genome-scale Perturb-seq dataset of primary human CD4+ T cells~\citep{zhu2025genome}, downloaded from the CZI Virtual Cell Models repository. Per-perturbation log-fold changes against non-targeting controls were computed at both the 8-hour and 48-hour time points across 4{,}000 highly variable genes. Input embeddings were constructed by applying PCA to the 8-hour log-fold changes and retaining the top 70 principal components; the 48-hour log-fold changes serve as the prediction target.

\paragraph{Surrogate model.} We use an MLP with a single hidden layer and ReLU activation, followed by a linear output layer that maps the 70-dimensional
input embedding to a 4,000-dimensional log-fold change vector. The model is trained with SGD (batch size 64) for 600 epochs using a learning rate of 1.0 with a linear scheduler lower-bounded at 0.1. We select the best checkpoint based on validation loss. Ground-truth labels for each of the 4{,}000 highly variable genes are obtained by thresholding absolute log-fold changes, $|\mathrm{LFC}| > 0.5$.

\paragraph{Portfolio.} Following similar logic as in Appendix \ref{app:reg-setup}, we choose only Random, CoreSet, ProbCover and TypiClust for our portfolio in these tasks.

\clearpage
\section{Performance Across Setups}
\label{app:perf-setups}

The seven setups report different performance metrics (accuracy, MSE, and AUPRC) which are not directly comparable, making averaging raw performance differences across setups meaningless. For a method with terminal performance $p$ on a given setup, we therefore define its relative gain over Random as
\begin{equation}
g = \frac{p - p_{\text{random}}}{p_{\text{best}} - p_{\text{random}}},
\end{equation}
where $p_{\text{random}}$ and $p_{\text{best}}$ denote the terminal performance of Random and of the best pure strategy on that setup. By construction, $g = 0$ corresponds to matching Random and $g = 1$ to matching the best pure strategy in the portfolio. The gap from $g$ to $1$ measures the cost of not knowing in advance which strategy is best, and $g < 0$ indicates performance worse than random sampling.

We report the mean, worst-case, and rank summaries across the seven setups in Table~\ref{rand-impr-results}. Ranks are computed per setup among the seven methods and then averaged, with ties resolved from unrounded performance values.

\begin{table}[h]
  \caption{Relative gain and rank among the six adaptive
methods and the equal split control, averaged across the seven setups. Mean Rank and Worst Rank report the average and maximum rank among the 7 methods (rank 1 = best, computed per setup then averaged/maximized across setups). Median Rank corresponds to the median value of rank across the setups. Mean Gain and Rank values are mean $\pm$ SE across setups.}
  \label{rand-impr-results}
  \centering
  \small
  \tabcolsep=0.11cm
  \begin{tabular}{lccccc}
    \toprule
    Method & Mean Gain ($\uparrow$)  & Worst Gain ($\uparrow$) & Mean Rank ($\downarrow$) & Median Rank ($\downarrow$) & Worst Rank ($\downarrow$)\\
    \midrule
    Equal Split & 0.230 \se{0.214} & -0.888 & 4.29 \se{0.77} & 4 & 7\\
    ALBL & 0.150 \se{0.227} & -1.055 & 5.43 \se{0.27} & 6 & 6\\
    UCB & 0.228 \se{0.214} & -1.055 & 4.14 \se{0.55} & 4 & 6\\
    SCRiBLe & 0.266 \se{0.208} & -0.944 & 4.14 \se{0.47} & 4 & 6\\
    SelectAL & 0.370 \se{0.178} & -0.611 & 3.14 \se{0.43} & 3 & 5\\
    AutoAL & 0.027 \se{0.251} & -1.055 & 5.71 \se{0.77} & 7 & 7\\
    FractAL & \textbf{0.689 \se{0.096}} & \textbf{0.222} & \textbf{1.14 \se{0.13}} & \textbf{1} & \textbf{2}\\
    \bottomrule
  \end{tabular}
\end{table}

FractAL has the highest mean gain (0.689) and the only positive \textit{worst-case} gain (0.222): every other method, including the non-adaptive Equal Split control, has at least one setup where its terminal performance falls below Random's, with AutoAL, ALBL, and UCB falling short by more than the entire Random-to-best span in their weakest setup. FractAL is also the top-ranked method on average (mean rank 1.14 of 7, median 1, against 3.14 and median 3 for the next closest method, SelectAL). We also note that no method attains $g > 1$ on any setup (see also Appendix~\ref{app:precision_k}). These results indicate that FractAL's advantage reflects consistent competitiveness across the full range of setups we evaluate.

\clearpage
\section{Per-Setup Results with Pure Strategies Performance}
\label{app:complete-tables}
We report here the performance of every pure strategy in each setup's portfolio, alongside the one from all of the adaptive methods (Tables \ref{tab:clas-all-pure}, \ref{tab:reg-all-pure} and \ref{tab:bio-all-pure}). Note that no adaptive method has prior knowledge of which strategies perform well, and each strategy selection method must infer it during the run from its reward signal. The best and worst pure strategies therefore serve as a reference range rather than as competing methods. In each table we bold the best pure strategy and the best strategy selection method, marking ties where they occur. Across all setups, FractAL achieves the highest test performance among the strategy selection methods, and either matches the best pure strategy or ranks second in the portfolio.

\begin{table}[h]
    \captionof{table}{Accuracy ($\uparrow$) on CIFAR-10 at the end of $T=10$ rounds.}
  \label{tab:clas-all-pure}
  \centering
  \small
  \tabcolsep=0.20cm
  \begin{tabular}{ccc}
    \toprule
    Method  & (C1): CIFAR-10 ($K=3$)  & (C2): CIFAR-10 ($K=8$)\\
    \midrule
    \multicolumn{3}{c}{\textbf{Individual Strategies}}\\
    \midrule
    Random &  69.4 \se{0.1} & 69.4 \se{0.1}\\
    CoreSet & 58.7 \se{0.2} & 58.7 \se{0.2}\\
    ProbCover & \textbf{71.2 \se{0.1}} & \textbf{71.2 \se{0.1}}\\
    BADGE & 69.6 \se{0.2} & 69.6 \se{0.2}\\
    LeastConf & 69.6 \se{0.1} & 69.6 \se{0.1}\\
    MaxEnt & 69.5 \se{0.1} & 69.5 \se{0.1}\\
    MinMargin & 69.7 \se{0.1} & 69.7 \se{0.1}\\
    TypiClust & 68.8 \se{0.1} & 68.8 \se{0.1}\\
    \midrule
        \multicolumn{3}{c}{\textbf{Strategy Selection Methods}}\\
\midrule
    Equal Split  & 67.8 \se{0.1} &69.1 \se{0.1}\\
    UCB & 67.5 \se{0.1} & 69.5 \se{0.1}\\
    ALBL & 67.5 \se{0.5} & 68.6 \se{0.3}\\
    SCRiBLe & 67.7 \se{0.2} & 69.2 \se{0.1}\\
    SelectAL & 68.3 \se{0.1} & 69.4 \se{0.1}\\
    AutoAL & 67.5 \se{0.1} & 68.3 \se{0.1}\\
    FractAL  & \textbf{70.2 \se{0.2}} & \textbf{69.8 \se{0.2}}\\
    \bottomrule
  \end{tabular}
\end{table}

\begin{table}[b]
  \caption{Test MSE ($\times 10^{-1}$, $\downarrow$) on tabular regression datasets at the end of $T=10$ rounds.}
  \label{tab:reg-all-pure}
  \centering
  \small
  \tabcolsep=0.15cm
  \begin{tabular}{cccc}
    \toprule
    Method  & KEGG  & SARCOS & Diamonds\\
    \midrule
    \multicolumn{4}{c}{\textbf{Individual Strategies}}\\
    \midrule
    Random & 0.79 \se{0.02} & 0.79 \se{0.01} & 0.11 \se{0.01}\\
    CoreSet & 0.53 \se{0.01} & 0.73 \se{0.01} & \textbf{0.08 \se{0.01}} \\
    ProbCover & \textbf{0.49 \se{0.00}} & \textbf{0.72 \se{0.01}} & 0.12 \se{0.01}\\
    TypiClust & 0.97 \se{0.01} & 0.77 \se{0.01} & 0.10 \se{0.01}\\
    \midrule
        \multicolumn{4}{c}{\textbf{Strategy Selection Methods}}\\
\midrule
    Equal Split & 0.62 \se{0.01} & 0.76 \se{0.01} & \textbf{0.08 \se{0.01}}\\
    UCB & 0.61 \se{0.01} & 0.77 \se{0.01} & 0.10 \se{0.01}\\
    ALBL & 0.63 \se{0.02} & 0.77 \se{0.01} & 0.09 \se{0.01}\\
    SCRiBLe & 0.63 \se{0.01} & 0.75 \se{0.01} & 0.09 \se{0.01}\\
    SelectAL &  0.59 \se{0.01} & 0.76 \se{0.01} & 0.09 \se{0.01}\\
    AutoAL   &  0.65 \se{0.01} & 0.82 \se{0.01} & 0.10 \se{0.01}\\
    FractAL & \textbf{0.58 \se{0.01}} & \textbf{0.74 \se{0.01}} & \textbf{0.08 \se{0.01}}\\
    \bottomrule
  \end{tabular}
\end{table}

\begin{table}[t]
  \caption{AUPRC ($\uparrow$) for high-throughput perturbation datasets after $T=6$ rounds.}
  \label{tab:bio-all-pure}
  \centering
  \small
  \tabcolsep=0.11cm
  \begin{tabular}{ccc}
    \toprule
    Method  & \textbf{B1:} BMDM  & \textbf{B2:} CD4+ T Cell\\
    \midrule
    \multicolumn{3}{c}{\textbf{Individual Strategies}}\\
    \midrule
    Random & 0.4165 \se{0.0011} & 0.3403 \se{0.0045}\\
    CoreSet & \textbf{0.4377 \se{0.0005}} & \textbf{0.3696 \se{0.0002}} \\
    ProbCover & 0.4305 \se{0.0004} & 0.3441 \se{0.0003}\\
    TypiClust & 0.4301 \se{0.0006} & 0.3413 \se{0.0061} \\
    \midrule
    \multicolumn{3}{c}{\textbf{Strategy Selection Methods}}\\
    \midrule
    Equal Split & 0.4279 \se{0.0006} & 0.3442 \se{0.0065}\\
    UCB & 0.4323 \se{0.0013} & 0.3590 \se{0.0015}\\
    ALBL & 0.4289 \se{0.0018} & 0.3548 \se{0.0021}\\
    SCRiBLe & 0.4299 \se{0.0006} & 0.3557 \se{0.0006}\\
    SelectAL  &  0.4346 \se{0.0010} & 0.3579 \se{0.0012}\\
    AutoAL  &  0.4336 \se{0.0005} & 0.3603 \se{0.0008}\\
    FractAL & \textbf{0.4363 \se{0.0006}} & \textbf{0.3642 \se{0.0013}}\\
    \bottomrule
  \end{tabular}
\end{table}

\clearpage
\section{Ablation Studies}
\label{app-ablation}
We perform the ablation experiments for our proposed FractAL method. In particular, we test the sensitivity of the method to the EWMA smoothing parameter, the contribution of the mirror descent update, the utility of the influence-based rewards as compared to retraining based rewards and the importance of soft fractional allocation among portfolio strategies. We detail the results of each of these experiments below.

\subsection{Smoothing Parameter}
We sweep over $\alpha$ values for across all setups and report the terminal accuracy. The numbers are averaged over 5 splits and 10 runs per split. Table~\ref{abl-ewma} contains the accuracy numbers for different values of the smoothing parameter $\alpha$. The similar performance across different $\alpha$ values highlights the robustness of the method to this hyperparameter. No Smoothing (i.e. $\alpha = 1.0$) performs particularly worse on two out of seven setups, CD4+ T cells and Diamonds.

\begin{table}[h]
    \caption{Terminal Performance of FractAL for different values of the EWMA smoothing parameter $\alpha$ on all seven setups. MSE numbers represent (MSE $\times 10^{-1}$)}
  \label{abl-ewma}
  \centering
  \footnotesize
  \tabcolsep=0.1cm
  \begin{tabular}{lccccccc}
    \toprule
    Dataset  & Metric  & $\alpha=0.1$ & $\alpha=0.3$ & $\alpha=0.5$ & $\alpha=0.7$ & $\alpha=1.0$ \\[0.1cm]
    \midrule
    CIFAR-10 (C1) & Acc ($\uparrow$) & 70.0 \se{0.2} & 70.1 \se{0.2} & 70.2 \se{0.2} & 70.0 \se{0.2} & 70.2 \se{0.1}\\
    CIFAR-10 (C2) & Acc ($\uparrow$) & 70.0 \se{0.2} & 69.9 \se{0.1} & 69.8 \se{0.2} & 70.1 \se{0.2} & 69.9 \se{0.1}\\
    KEGG & MSE ($\downarrow$) & 0.60 \se{0.01} & 0.60 \se{0.01} & 0.58 \se{0.01} & 0.58 \se{0.01} & 0.59 \se{0.01}\\
    SARCOS & MSE ($\downarrow$) & 0.73 \se{0.01} & 0.76 \se{0.01} & 0.74 \se{0.01} & 0.74 \se{0.01} & 0.74 \se{0.01}\\
    Diamonds & MSE ($\downarrow$) & 0.09 \se{0.01} & 0.09 \se{0.01} & 0.08 \se{0.01} & 0.09 \se{0.01} & 0.10 \se{0.01}\\
    BMDMs  & AUPRC ($\uparrow$) & 0.4348 \se{0.0008} & 0.4356 \se{0.0006} & 0.4363 \se{0.0006} & 0.4360 \se{0.0008} & 0.4365 \se{0.0008}\\
    CD4+ T cells & AUPRC ($\uparrow$) & 0.3629 \se{0.0013} & 0.3630 \se{0.0013} & 0.3642 \se{0.0013} & 0.3644 \se{0.0013} & 0.3593 \se{0.0050}\\
    \bottomrule
  \end{tabular}
\end{table}

\subsection{Allocation Optimization Algorithm}
We replace the mirror descent allocation update with the following additive update with clipping and renormalization:
\begin{equation}
\begin{aligned}
    f^t_k &= \max\!\left(q_k^t + \lambda\,\eta^t\,\bar{r}^t_k,\; p_{\min}\right), \\
    q_k^{t+1} &= \frac{f^t_k}{\sum_{j=1}^K f^t_j},
\end{aligned}
\end{equation}
where we use a fixed $\lambda=0.15$. $p_k^{t+1}$ is then computed as in Equation~\ref{meth-pcal}. We benchmark this approach against FractAL across all setups (Table~\ref{linear-desc-results}). The mirror-descent update is never worse, and better on five of seven datasets, supporting its use for navigating the allocation simplex over an additive update with clipping and renormalization.

\begin{table}[h]
\caption{Performance comparison of mirror descent vs additive update with clipping and renormalization across all setups.}
  \label{linear-desc-results}
  \centering
  \small
  \tabcolsep=0.11cm
  \begin{tabular}{lccc}
    \toprule
    Dataset  & Metric  & Mirror Desc. & Additive (clip + renorm.)\\[0.1cm]
    \midrule
    CIFAR-10 (C1) & Acc ($\uparrow$) & \textbf{70.2 \se{0.2}} & 69.8 \se{0.2} \\
    CIFAR-10 (C2) & Acc ($\uparrow$) & \textbf{69.8 \se{0.2}} & \textbf{69.8 \se{0.1}}\\
    KEGG & MSE ($\times 10^{-1}$, $\downarrow$) & \textbf{0.58 \se{0.01}} & 0.61 \se{0.01}\\
    SARCOS & MSE ($\times 10^{-1}$, $\downarrow$) & \textbf{0.74 \se{0.01}} & 0.76 \se{0.01}\\
    Diamonds & MSE ($\times 10^{-1}$, $\downarrow$) & \textbf{0.08 \se{0.01}} & 0.09 \se{0.01}\\
    BMDMs  & AUPRC ($\uparrow$) & \textbf{0.4363 \se{0.0006}} & 0.4299 \se{0.0005}\\
    CD4+ T cells & AUPRC ($\uparrow$) & \textbf{0.3642 \se{0.0013}} & 0.3533 \se{0.0051}\\
    \bottomrule
  \end{tabular}
\end{table}

\subsection{Influence Reward Attribution}
\label{sub-change-infl}
We test the utility of the influence rewards by swapping them with changes in performance metrics on the validation set after retraining the model (keeping the fractional allocation framework unchanged). More specifically, the retraining reward for a strategy is the difference in performance on validation set with and without the strategy's latest sub-batch of chosen points. We perform the experiment across all our setups and report the numbers averaged over 5 splits and 10 seeds per split (Table~\ref{infl-retr-results}).

\begin{table}[h]
\caption{Performance comparison with influence vs retraining based reward across all setups.}
  \label{infl-retr-results}
  \centering
  \small
  \tabcolsep=0.11cm
  \begin{tabular}{lccc}
    \toprule
    Dataset  & Metric  & Influence & Retraining\\[0.1cm]
    \midrule
    CIFAR-10 (C1) & Acc ($\uparrow$) & 70.2 \se{0.2} & \textbf{70.5 \se{0.1}}\\
    CIFAR-10 (C2) & Acc ($\uparrow$) & 69.8 \se{0.2} & \textbf{70.2 \se{0.2}}\\
    KEGG & MSE ($\times 10^{-1}$, $\downarrow$) & \textbf{0.58 \se{0.01}} & 0.59 \se{0.01}\\
    SARCOS & MSE ($\times 10^{-1}$, $\downarrow$) & \textbf{0.74 \se{0.01}} & \textbf{0.74 \se{0.01}}\\
    Diamonds & MSE ($\times 10^{-1}$, $\downarrow$) & \textbf{0.08 \se{0.01}} & \textbf{0.08 \se{0.01}}\\
    BMDMs  & AUPRC ($\uparrow$) & 0.4363 \se{0.0006} & \textbf{0.4383 \se{0.0006}}\\
    CD4+ T cells & AUPRC ($\uparrow$) & \textbf{0.3642 \se{0.0013}} & 0.3514 \se{0.0006}\\
    \bottomrule
  \end{tabular}
\end{table}

The influence reward is on par with the retraining-based reward: it trails by at most 0.4 accuracy points on CIFAR-10 and matches or outperforms on four of the other five tasks. However, in runtime cost, influence-based AL run is overall $2.5\times$ faster than the retraining based AL run.

To further establish influence rewards as a sound substitute for retraining reward, we also study the correlation between the two rewards across rounds for the CIFAR-10 (C1) setup. Specifically, we report the Spearman rank correlation between the two reward rankings as well as the agreement (in percentage) between the two rewards for the best and worst performing strategies. Table~\ref{infl-retr-corr-results} reports the aforementioned numbers. The rank correlation ($\pm$SE) ranges between 0.7 and 0.8, while the best and worst agreements are generally $>70\%$ indicating that the two rewards rank strategies consistently. Therefore, the influence estimator is a reliable proxy for strategy utility and hence a sound substitute for the more expensive retraining-based attribution within the proposed fractional allocation framework.

\begin{table}[h]
\caption{Correspondence between the retraining based reward and corresponding influence estimates across different active learning rounds for the CIFAR-10 (C1) setup.}\vspace{0.1cm}
  \label{infl-retr-corr-results}
  \centering
  \small
  \tabcolsep=0.11cm
  \begin{tabular}{lccc}
    \toprule
    Round  & Spearman Rank Correlation  & Best Agreement (\%) & Worst Agreement (\%)\\[0.1cm]
    \midrule
    2 & 0.70 \se{0.08} & 70.0 & 80.0\\
    3 & 0.80 \se{0.06} & 80.0 & 83.3\\
    4 & 0.78 \se{0.04} & 70.0 & 86.6\\
    5 & 0.73 \se{0.08} & 80.0 & 76.6\\
    6 & 0.73 \se{0.07} & 76.6 & 76.6\\
    \bottomrule
  \end{tabular}
\end{table}

\subsection{Soft Fractional Allocation}
While preserving the influence-based reward attribution, we replace the proposed \textit{soft} fractional allocation with a hard allocation mechanism in which the strategy with the highest reward receives the entire budget for the round. We call this variant FractAL (Hard). Soft allocation improves on five of seven setups and is comparable on the remaining two (Table~\ref{hard-soft-results}). Distributing the budget across strategies therefore outperforms committing it to the single highest-reward strategy, even when both use the same reward signal.

\begin{table}[h]
\caption{Performance comparison of hard versus soft allocation across all setups.}
  \label{hard-soft-results}
  \centering
  \small
  \tabcolsep=0.11cm
  \begin{tabular}{lccc}
    \toprule
    Dataset  & Metric  & FractAL & FractAL (Hard)\\[0.1cm]
    \midrule
    CIFAR-10 (C1) & Acc ($\uparrow$) & \textbf{70.2 \se{0.2}} & 69.3 \se{0.1}\\
    CIFAR-10 (C2) & Acc ($\uparrow$) & 69.8 \se{0.2} & \textbf{70.0 \se{0.1}}\\
    KEGG & MSE ($\times 10^{-1}$, $\downarrow$) & \textbf{0.58 \se{0.01}} & 0.60 \se{0.01}\\
    SARCOS & MSE ($\times 10^{-1}$, $\downarrow$) & \textbf{0.74 \se{0.01}} & \textbf{0.74 \se{0.01}}\\
    Diamonds & MSE ($\times 10^{-1}$, $\downarrow$) & \textbf{0.08 \se{0.01}} & 0.09 \se{0.01}\\
    BMDMs  & AUPRC ($\uparrow$) & \textbf{0.4363 \se{0.0006}} & 0.4309 \se{0.0010}\\
    CD4+ T cells & AUPRC ($\uparrow$) & \textbf{0.3642 \se{0.0013}} & 0.3592 \se{0.0017}\\
    \bottomrule
  \end{tabular}
\end{table}

\clearpage
\section{Detailed Runtime Analysis}
\label{app:runtime}
Table~\ref{walltime-results} reports wall-clock time for a single complete AL run on CIFAR-10 (C2, $K=8$) on one NVIDIA L40S GPU, decomposed into two phases: Acquisition, the time to query strategies in the portfolio and assemble the batch, and Allocation Update, the time to compute the reward signal and update the allocation for the next round.

\begin{table}[h]
\caption{Runtime (in minutes) of different components of a single AL run for different methods. We use the CIFAR10 (C2, $K=8$) setting for these measurements.}\vspace{0.1cm}
  \label{walltime-results}
  \centering
  \small
  \tabcolsep=0.11cm
  \begin{tabular}{lccc}
    \toprule
    Method  & Acquisition  & Allocation Update & Total\\[0.1cm]
    \midrule
    ALBL & 1.2 & 7.3 $\times 10^{-5}$ & 1.2\\
    UCB & 1.5 & 7.2 $\times 10^{-5}$ & 1.5\\
    SCRiBLe & 8.8 & 4.6 $\times 10^{-4}$ & 8.8\\
    SelectAL  &  3.3 & 19.8 & 23.1\\
    AutoAL  &  20.7 & 28.5 & 49.2\\
    FractAL & 8.8 & 0.5 & 9.3\\
    \bottomrule
  \end{tabular}
\end{table}

\clearpage
\section{Additional Budget Allocation Plots}
\label{app:more-plots}
For completeness, we provide the performance plots for pure strategies in the portfolio and budget allocation evolution plots for each of our methods and baselines on all our setups. We often notice that FractAL increases the allocation to the empirically best strategies across all setups, while reducing the share of the suboptimal ones in the portfolio. Detailed analysis for each plot is included in its respective caption.

All allocation and performance plots (Figures~\ref{fig:cf10_3comp_frac_evolution}--\ref{fig:tcell_frac_evolution}) show rounds $t=2$ to $t=T-1$. We omit
$t=0$, which corresponds to acquiring $\mathcal{D}_{\text{init}}$, and $t=1$, at which every method, including FractAL, starts from an equal split.

\begin{figure}[h]
  \centering
  \begin{subfigure}[t]{0.24\textwidth}
    \centering
    \includegraphics[width=\textwidth]{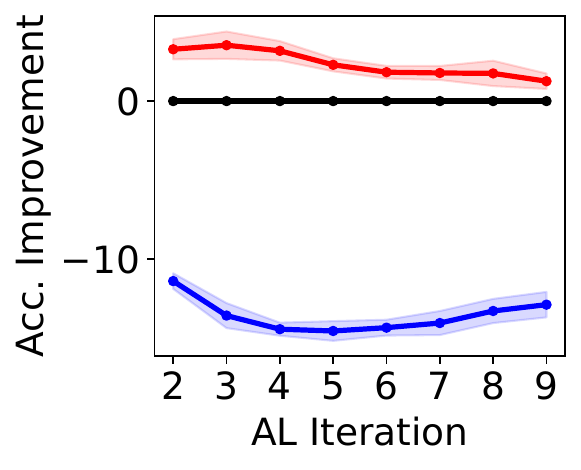}
    \caption{Test Performance}
  \end{subfigure}
  \hfill
  \begin{subfigure}[t]{0.24\textwidth}
    \centering
    \includegraphics[width=\textwidth]{figures/more_size_plots/cf10_3comp/cf10_3comp_3comp_FractAL_strat_evolution.pdf}
    \caption{FractAL}
  \end{subfigure}
  \hfill
  \begin{subfigure}[t]{0.24\textwidth}
    \centering
    \includegraphics[width=\textwidth]{figures/more_size_plots/cf10_3comp/cf10_3comp_3comp_SelectAL_strat_evolution.pdf}
    \caption{SelectAL}
  \end{subfigure}
  \vspace{2ex} % Vertical spacing between rows
  \begin{subfigure}[t]{0.24\textwidth}
    \centering
    \includegraphics[width=\textwidth]{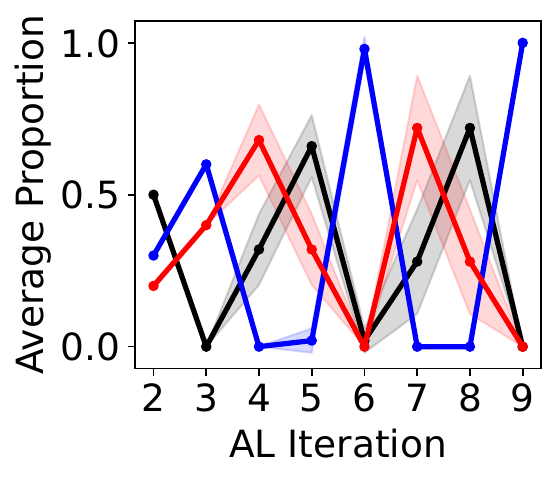}
    \caption{UCB}
  \end{subfigure}
  \begin{subfigure}[t]{0.24\textwidth}
    \centering
    \includegraphics[width=\textwidth]{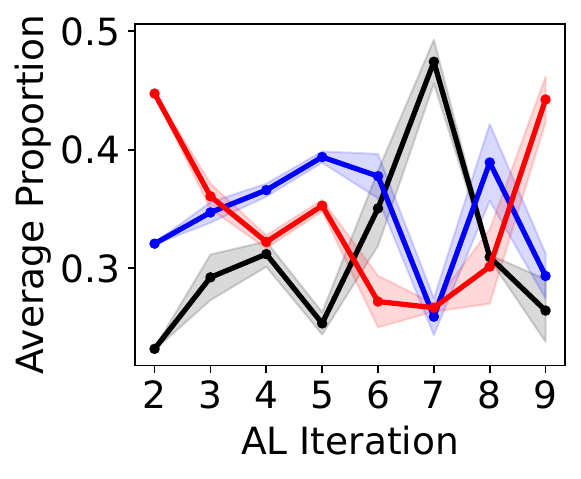}
    \caption{SCRiBLe}
  \end{subfigure}
  \begin{subfigure}[t]{0.24\textwidth}
    \centering
    \includegraphics[width=\textwidth]{figures/more_size_plots/cf10_3comp/cf10_3comp_3comp_ALBL_strat_evolution.pdf}
    \caption{ALBL}
  \end{subfigure}
  \begin{subfigure}[t]{0.24\textwidth}
    \centering
    \includegraphics[width=\textwidth]{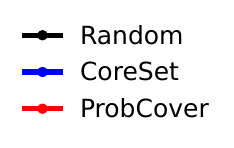}
  \end{subfigure}
  \caption{Setup (C1): CIFAR-10 with $K=3$. (a) Test performance of pure strategies. (b--f) Average per-round budget allocation to Random, CoreSet, and ProbCover under each method. FractAL drives CoreSet to zero and concentrates on ProbCover. SelectAL allocates roughly equal shares to Random and CoreSet, while bandit baselines (d--f) oscillate between strategies without converging. Lines show the mean and shaded bands show $\pm$ 1 standard deviation. For test accuracy, standard deviation is computed across seeds after averaging over splits; for allocations, across splits after averaging over seeds.}
  \label{fig:cf10_3comp_frac_evolution}
\end{figure}

\begin{figure}[t]
  \centering
  \begin{subfigure}[t]{0.24\textwidth}
    \centering
    \includegraphics[width=\textwidth]{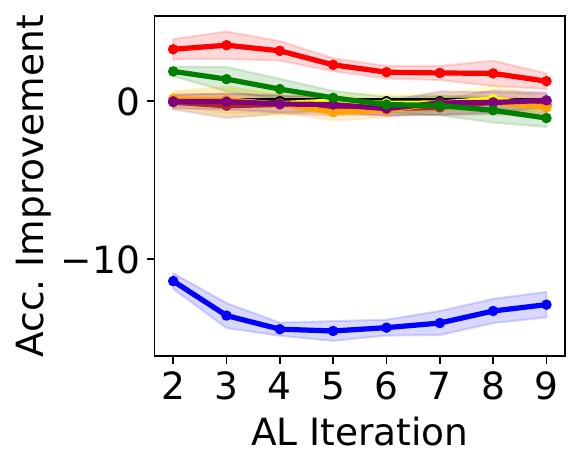}
    \caption{Test Performance}
  \end{subfigure}
  \hfill
  \begin{subfigure}[t]{0.24\textwidth}
    \centering
    \includegraphics[width=\textwidth]{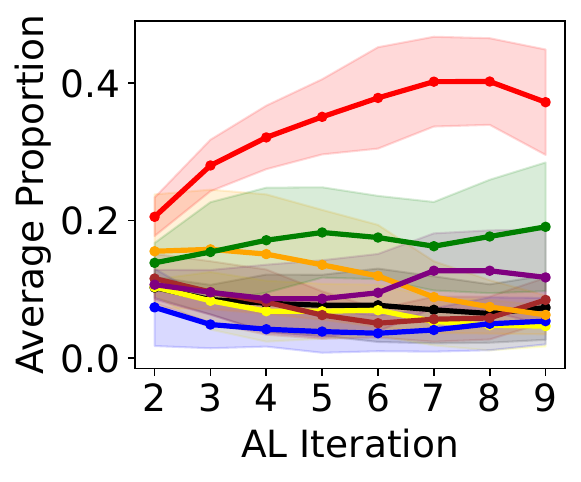}
    \caption{FractAL}
  \end{subfigure}
  \hfill
  \begin{subfigure}[t]{0.24\textwidth}
    \centering
    \includegraphics[width=\textwidth]{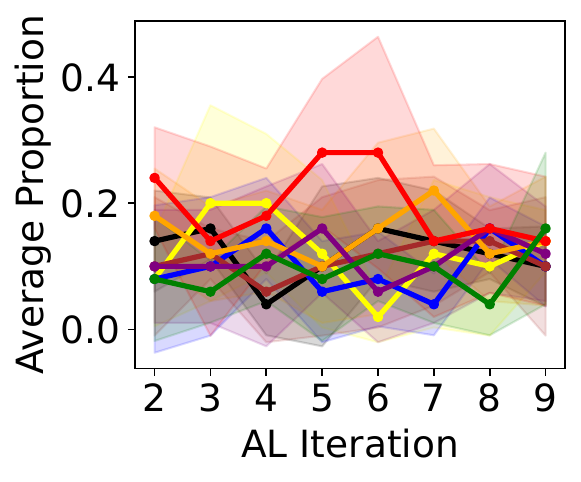}
    \caption{SelectAL}
  \end{subfigure}
  \vspace{2ex} % Vertical spacing between rows
  \begin{subfigure}[t]{0.24\textwidth}
    \centering
    \includegraphics[width=\textwidth]{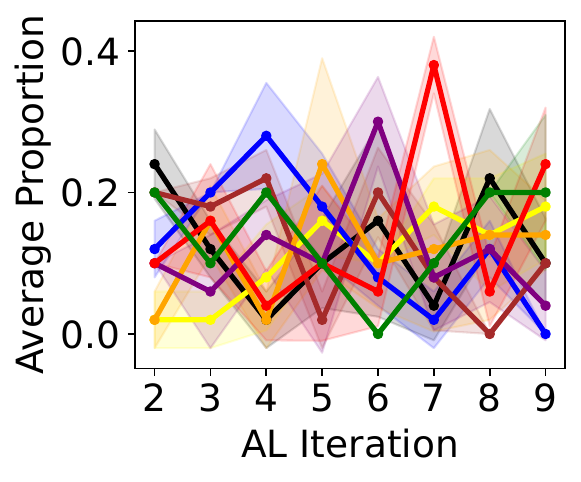}
    \caption{UCB}
  \end{subfigure}
  \begin{subfigure}[t]{0.24\textwidth}
    \centering
    \includegraphics[width=\textwidth]{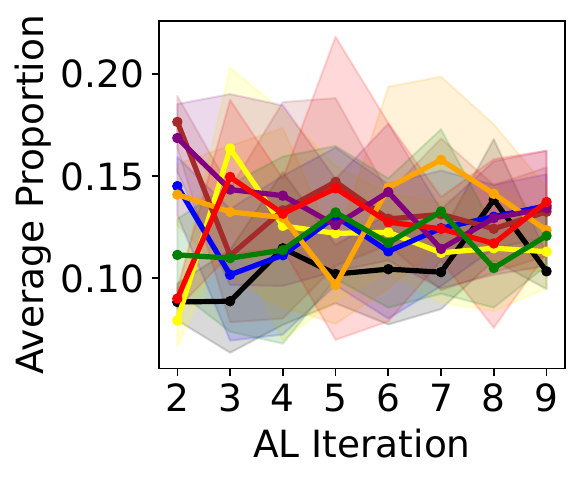}
    \caption{SCRiBLe}
  \end{subfigure}
  \begin{subfigure}[t]{0.24\textwidth}
    \centering
    \includegraphics[width=\textwidth]{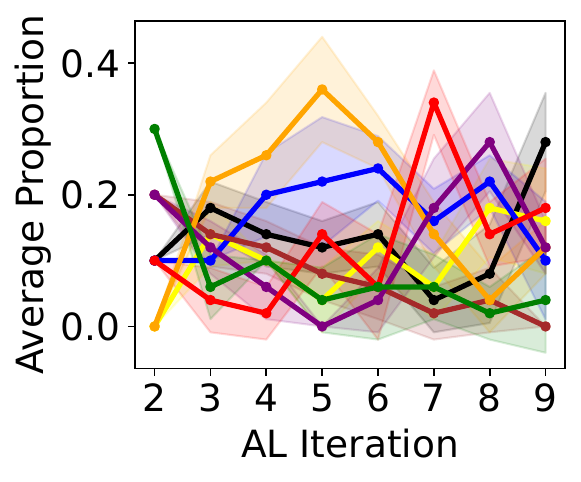}
    \caption{ALBL}
  \end{subfigure}
    \begin{subfigure}[t]{0.24\textwidth}
    \centering
    \includegraphics[width=0.7\textwidth]{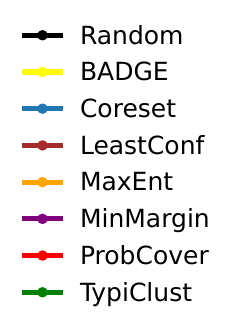}
    
  \end{subfigure}
  \caption{Setup (C2): CIFAR-10 with $K=8$. (a) Test performance of the 8 pure strategies. (b--f) Average per-round budget allocation to all 8 strategies under each method. FractAL increases the share of ProbCover and decreases CoreSet over rounds. Lines show the mean and shaded bands show $\pm$ 1 standard deviation. For test accuracy, standard deviation is computed across seeds after averaging over splits; for allocations, across splits after averaging over seeds.}
  \label{fig:cf10_8comp_frac_evolution}
\end{figure}

\begin{figure}[t]
  \centering
  \begin{subfigure}[t]{0.24\textwidth}
    \centering
    \includegraphics[width=\textwidth]{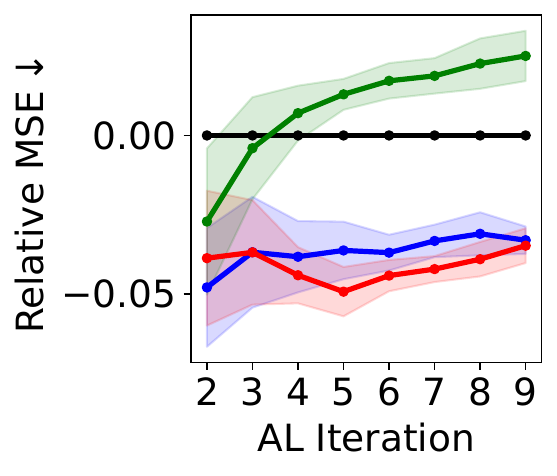}
    \caption{Test Performance}
  \end{subfigure}
  \hfill
  \begin{subfigure}[t]{0.24\textwidth}
    \centering
    \includegraphics[width=\textwidth]{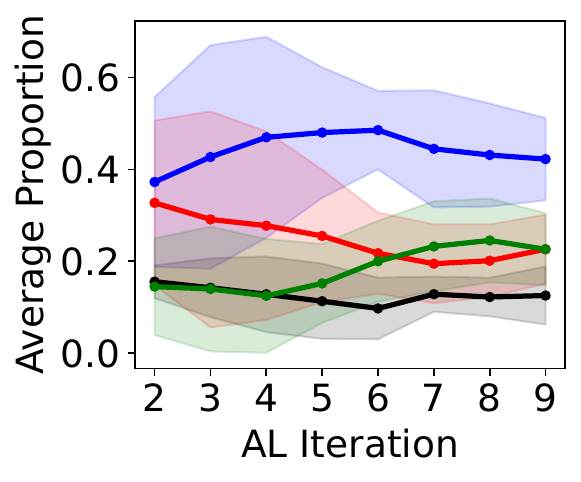}
    \caption{FractAL}
  \end{subfigure}
  \hfill
  \begin{subfigure}[t]{0.24\textwidth}
    \centering
    \includegraphics[width=\textwidth]{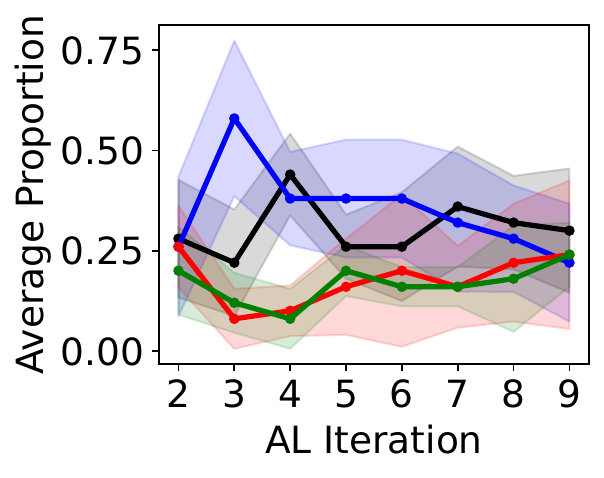}
    \caption{SelectAL}
  \end{subfigure}
  \vspace{2ex} % Vertical spacing between rows
  \begin{subfigure}[t]{0.24\textwidth}
    \centering
    \includegraphics[width=\textwidth]{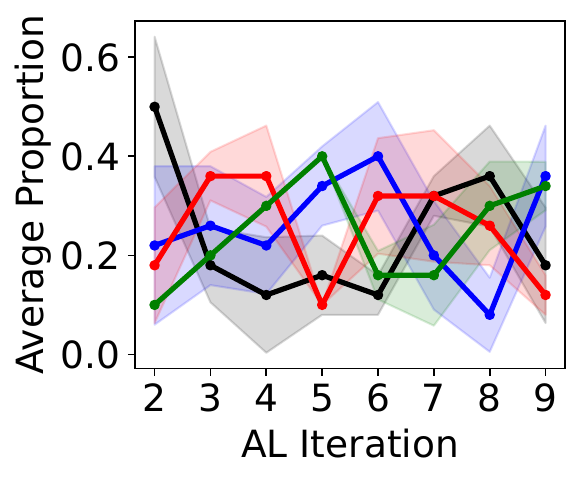}
    \caption{UCB}
  \end{subfigure}
  \begin{subfigure}[t]{0.24\textwidth}
    \centering
    \includegraphics[width=\textwidth]{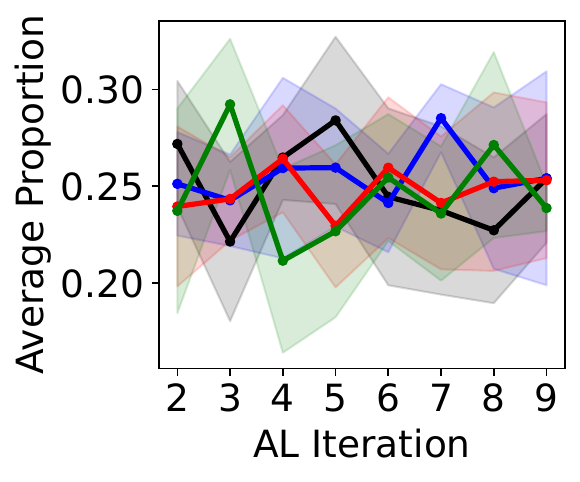}
    \caption{SCRiBLe}
  \end{subfigure}
  \begin{subfigure}[t]{0.24\textwidth}
    \centering
    \includegraphics[width=\textwidth]{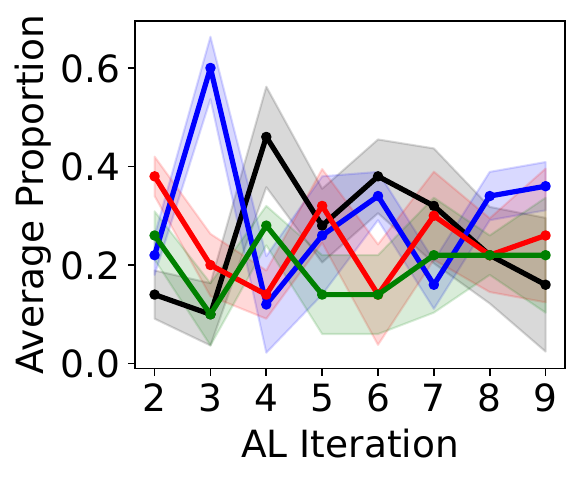}
    \caption{ALBL}
  \end{subfigure}
  \begin{subfigure}[t]{0.24\textwidth}
    \centering
    \includegraphics[width=\textwidth]{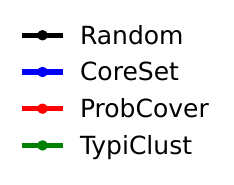}
  \end{subfigure}
\caption{KEGG metabolic reaction network experiment. (a) Test performance of the four pure strategies. (b--f) Average per-round budget allocation under each method. FractAL concentrates the budget on CoreSet (the second strongest strategy) while suppressing Random and TypiClust. SelectAL also upweights CoreSet but assigns ProbCover a lower share than Random consistently. The bandit baselines (d--f) oscillate between strategies without converging. Lines show the mean and shaded bands show $\pm$ 1 standard deviation. For test accuracy, standard deviation is computed across seeds after averaging over splits; for allocations, across splits after averaging over seeds.}
  \label{fig:kegg_frac_evolution}
\end{figure}

\begin{figure}[t]
  \centering
  \begin{subfigure}[t]{0.24\textwidth}
    \centering
    \includegraphics[width=\textwidth]{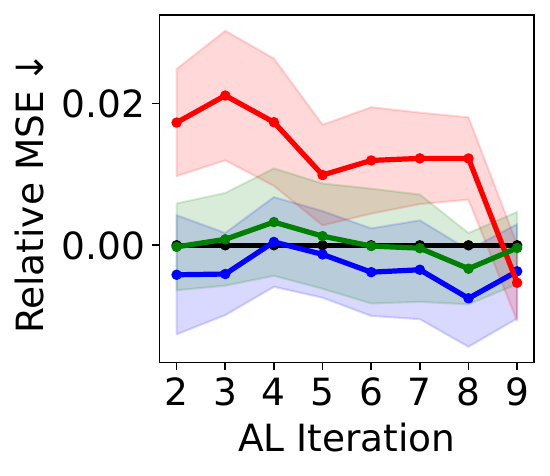}
    \caption{Test Performance}
  \end{subfigure}
  \hfill
  \begin{subfigure}[t]{0.24\textwidth}
    \centering
    \includegraphics[width=\textwidth]{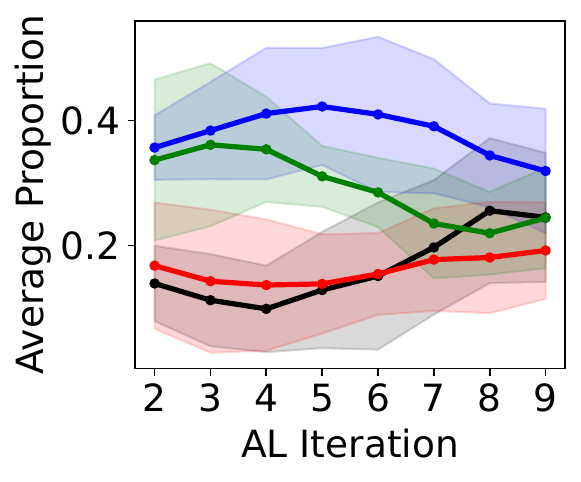}
    \caption{FractAL}
  \end{subfigure}
  \hfill
  \begin{subfigure}[t]{0.24\textwidth}
    \centering
    \includegraphics[width=\textwidth]{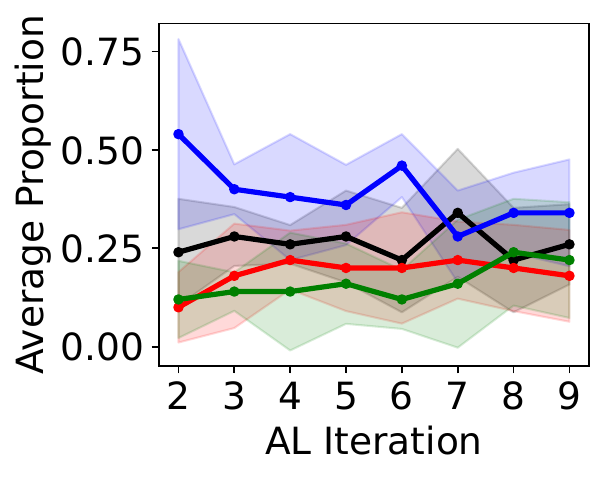}
    \caption{SelectAL}
  \end{subfigure}
  \vspace{2ex} % Vertical spacing between rows
  \begin{subfigure}[t]{0.24\textwidth}
    \centering
    \includegraphics[width=\textwidth]{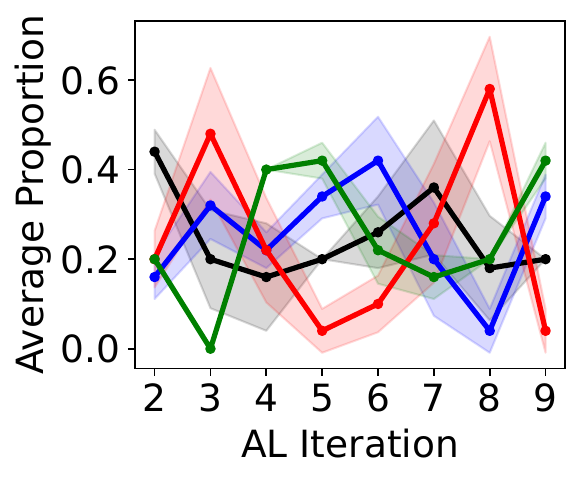}
    \caption{UCB}
  \end{subfigure}
  % \hspace{1em} 
  \begin{subfigure}[t]{0.24\textwidth}
    \centering
    \includegraphics[width=\textwidth]{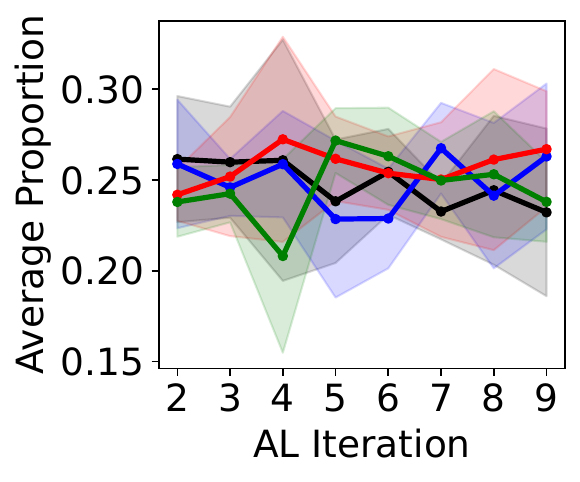}
    \caption{SCRiBLe}
  \end{subfigure}
  \begin{subfigure}[t]{0.24\textwidth}
    \centering
    \includegraphics[width=\textwidth]{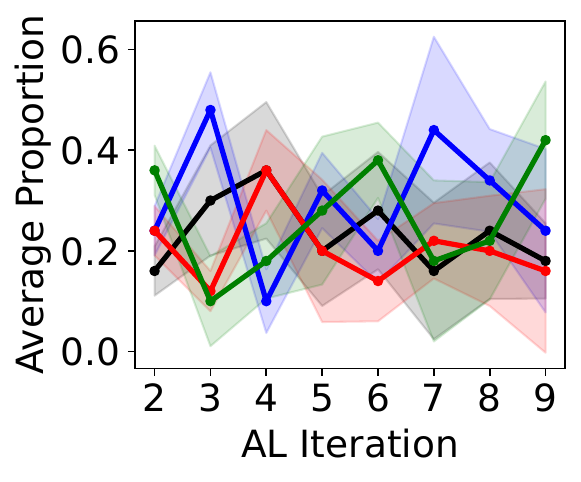}
    \caption{ALBL}
  \end{subfigure}
  \begin{subfigure}[t]{0.24\textwidth}
    \centering
    \includegraphics[width=\textwidth]{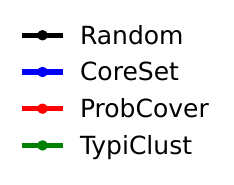}
  \end{subfigure}
\caption{SARCOS experiment. (a) Test performance of the four pure strategies. (b--f) Average per-round budget allocation under each method. FractAL and SelectAL upweight CoreSet (one of the strongest strategies) and suppress Random and ProbCover. SelectAL ranks TypiClust at the lowest allocation, despite it not being the weakest strategy; FractAL avoids this mistake. The bandit baselines (d--f) oscillate without converging. Lines show the mean and shaded bands show $\pm$ 1 standard deviation. For test accuracy, standard deviation is computed across seeds after averaging over splits; for allocations, across splits after averaging over seeds.}
  \label{fig:sarcos_frac_evolution}
\end{figure}

\begin{figure}[t]
  \centering
  \begin{subfigure}[t]{0.24\textwidth}
    \centering
    \includegraphics[width=\textwidth]{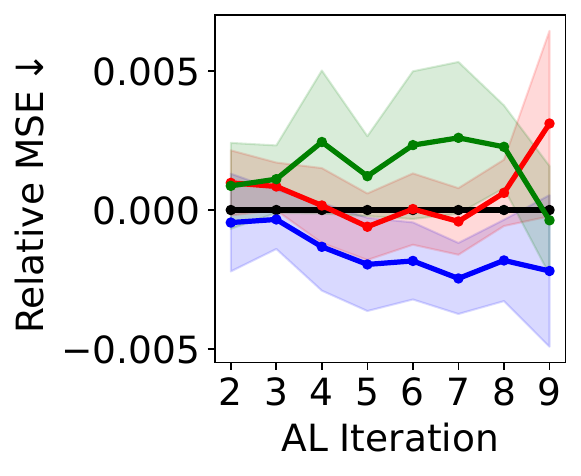}
    \caption{Test Performance}
  \end{subfigure}
  \hfill
  \begin{subfigure}[t]{0.24\textwidth}
    \centering
    \includegraphics[width=\textwidth]{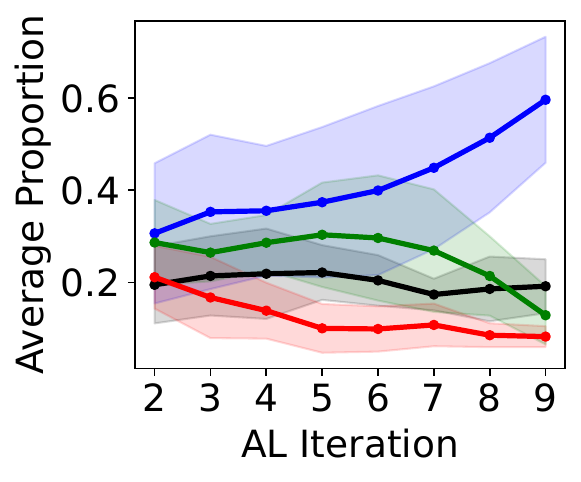}
    \caption{FractAL}
  \end{subfigure}
  \hfill
  \begin{subfigure}[t]{0.24\textwidth}
    \centering
    \includegraphics[width=\textwidth]{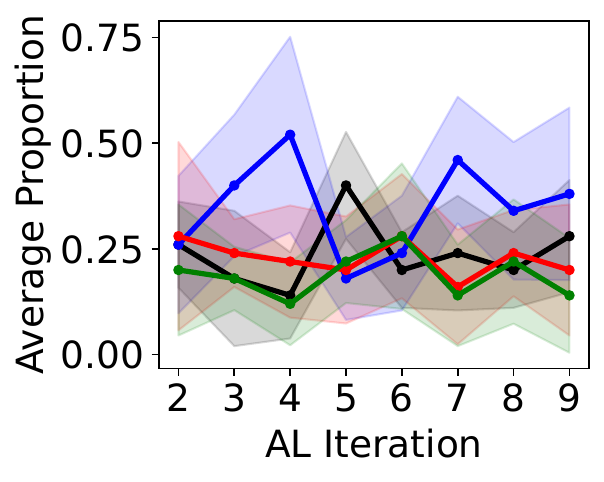}
    \caption{SelectAL}
  \end{subfigure}
  \vspace{2ex} % Vertical spacing between rows
  \begin{subfigure}[t]{0.24\textwidth}
    \centering
    \includegraphics[width=\textwidth]{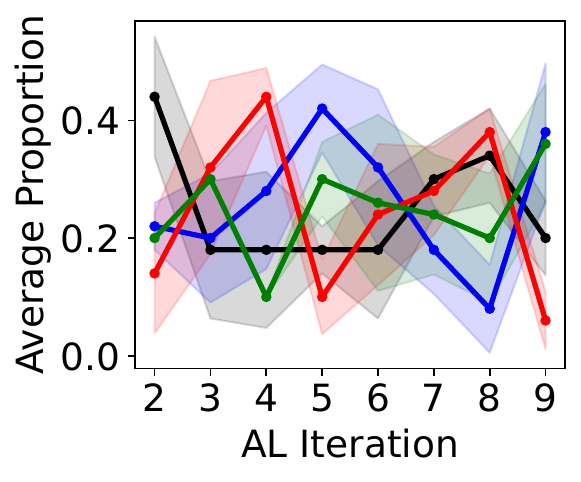}
    \caption{UCB}
  \end{subfigure}
  \begin{subfigure}[t]{0.24\textwidth}
    \centering
    \includegraphics[width=\textwidth]{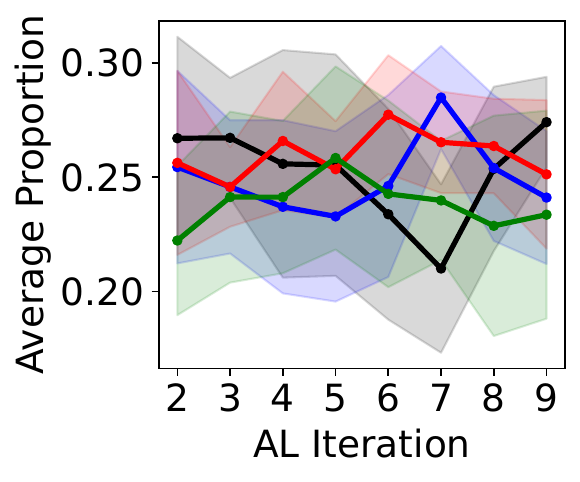}
    \caption{SCRiBLe}
  \end{subfigure}
  \begin{subfigure}[t]{0.24\textwidth}
    \centering
    \includegraphics[width=\textwidth]{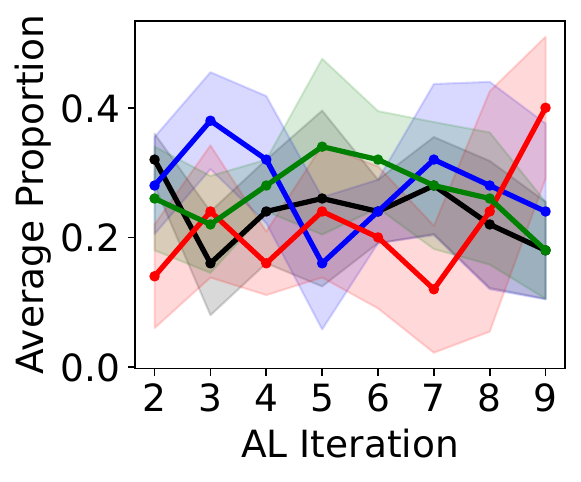}
    \caption{ALBL}
  \end{subfigure}
  \begin{subfigure}[t]{0.24\textwidth}
    \centering
    \includegraphics[width=\textwidth]{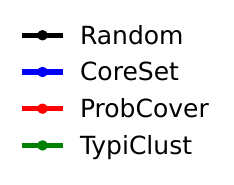}
  \end{subfigure}
  \caption{Diamonds experiment. (a) Test performance of the four pure strategies, which cluster within a narrow MSE range. (b--f) Average per-round budget allocation under each method. FractAL upweights CoreSet (the strongest individual strategy) and reduces the others, with the allocations aligned with the empirical ranking by the final rounds. The remaining methods (c--f) maintain roughly uniform allocations. Lines show the mean and shaded bands show $\pm$ 1 standard deviation. For test accuracy, standard deviation is computed across seeds after averaging over splits; for allocations, across splits after averaging over seeds.}
  \label{fig:dia_frac_evolution}
\end{figure}

\begin{figure}[t]
  \centering
  \begin{subfigure}[t]{0.24\textwidth}
    \centering
    \includegraphics[width=\textwidth]{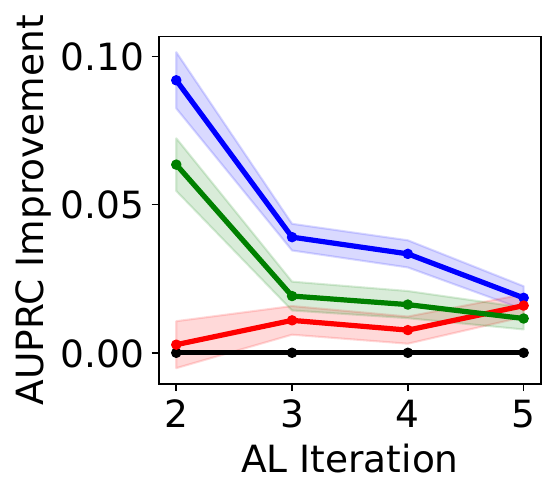}
    \caption{Test Performance}
  \end{subfigure}
  \hfill
  \begin{subfigure}[t]{0.24\textwidth}
    \centering
    \includegraphics[width=\textwidth]{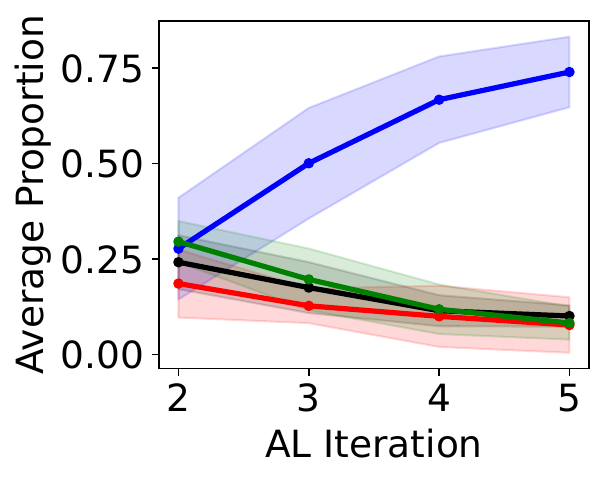}
    \caption{FractAL}
  \end{subfigure}
  \hfill
  \begin{subfigure}[t]{0.24\textwidth}
    \centering
    \includegraphics[width=\textwidth]{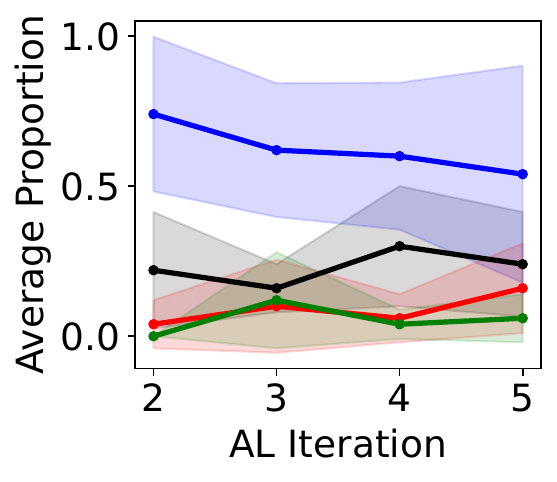}
    \caption{SelectAL}
  \end{subfigure}
  \vspace{2ex} % Vertical spacing between rows
  \begin{subfigure}[t]{0.24\textwidth}
    \centering
    \includegraphics[width=\textwidth]{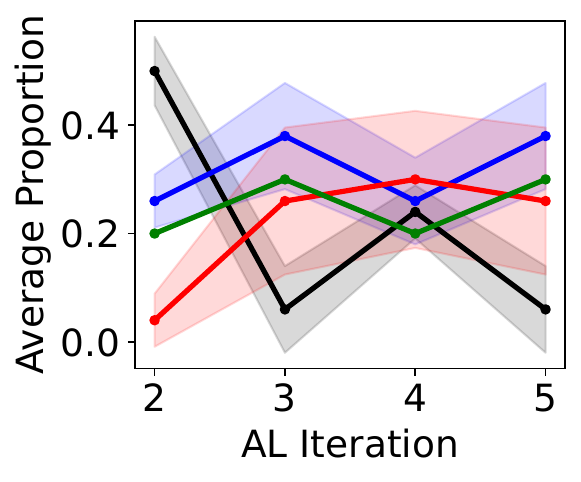}
    \caption{UCB}
  \end{subfigure}
  \begin{subfigure}[t]{0.24\textwidth}
    \centering
    \includegraphics[width=\textwidth]{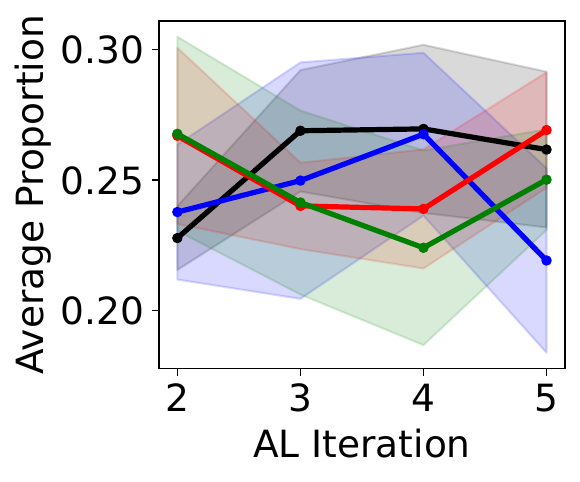}
    \caption{SCRiBLe}
  \end{subfigure}
  \begin{subfigure}[t]{0.24\textwidth}
    \centering
    \includegraphics[width=\textwidth]{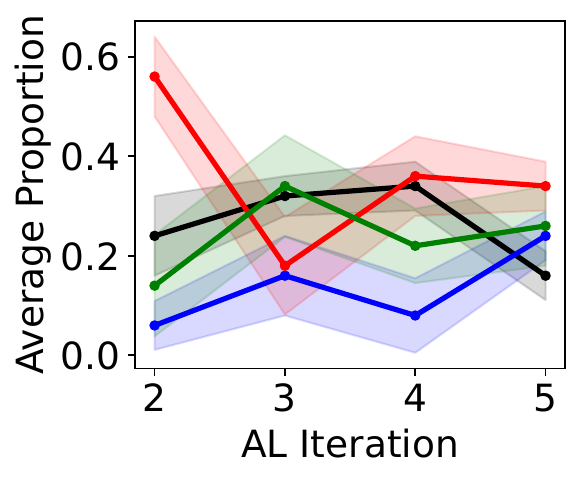}
    \caption{ALBL}
  \end{subfigure}
  \begin{subfigure}[t]{0.24\textwidth}
    \centering
    \includegraphics[width=\textwidth]{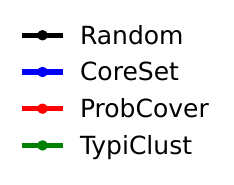}
  \end{subfigure}
\caption{Setup (B1): BMDM perturbation prediction. (a) Test AUPRC of the four pure strategies. (b--f) Average per-round budget allocation under each method. FractAL and SelectAL both upweight CoreSet (the strongest strategy) over rounds, but FractAL drives the other three strategies close to zero, whereas SelectAL retains substantial allocation to Random. The bandit baselines (d--f) oscillate between strategies without converging. Lines show the mean and shaded bands show $\pm$ 1 standard deviation. For test accuracy, standard deviation is computed across seeds after averaging over splits; for allocations, across splits after averaging over seeds.}
  \label{fig:pp_frac_evolution}
\end{figure}

\begin{figure}[t]
  \centering
  \begin{subfigure}[t]{0.24\textwidth}
    \centering
    \includegraphics[width=\textwidth]{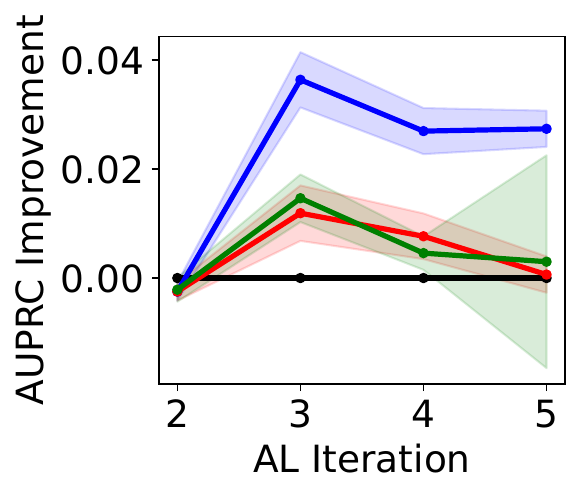}
    \caption{Test Performance}
  \end{subfigure}
  \hfill
  \begin{subfigure}[t]{0.24\textwidth}
    \centering
    \includegraphics[width=\textwidth]{figures/more_size_plots/bio_tcell/bio_tcell_4comp_FractAL_strat_evolution.pdf}
    \caption{FractAL}
  \end{subfigure}
  \hfill
  \begin{subfigure}[t]{0.24\textwidth}
    \centering
    \includegraphics[width=\textwidth]{figures/more_size_plots/bio_tcell/bio_tcell_4comp_SelectAL_strat_evolution.pdf}
    \caption{SelectAL}
  \end{subfigure}
  \vspace{2ex} % Vertical spacing between rows
  \begin{subfigure}[t]{0.24\textwidth}
    \centering
    \includegraphics[width=\textwidth]{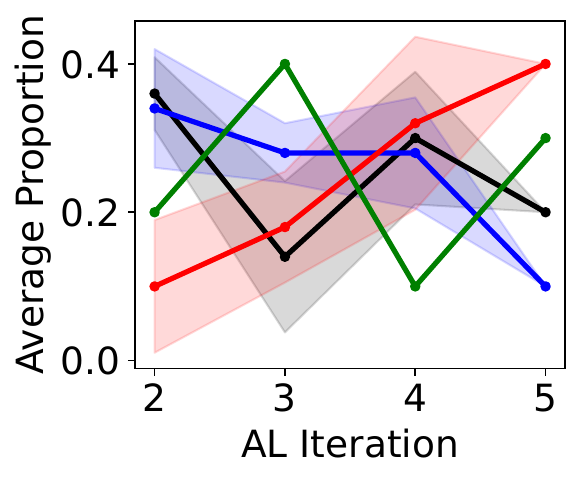}
    \caption{UCB}
  \end{subfigure}
  \begin{subfigure}[t]{0.24\textwidth}
    \centering
    \includegraphics[width=\textwidth]{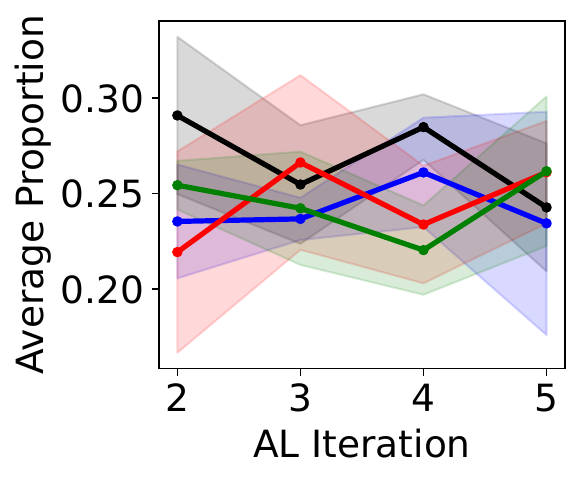}
    \caption{SCRiBLe}
  \end{subfigure}
  \begin{subfigure}[t]{0.24\textwidth}
    \centering
    \includegraphics[width=\textwidth]{figures/more_size_plots/bio_tcell/bio_tcell_4comp_ALBL_strat_evolution.pdf}
    \caption{ALBL}
  \end{subfigure}
  \begin{subfigure}[t]{0.24\textwidth}
    \centering
    \includegraphics[width=\textwidth]{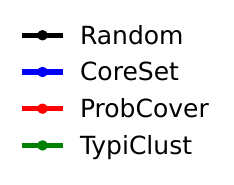}
  \end{subfigure}
  \caption{Setup (B2): CD4+ T cell perturbation prediction. (a) Test AUPRC of the four pure strategies. (b--f) Average per-round budget allocation under each method. FractAL consistently upweights CoreSet (the strongest strategy) and reduces the other three to similar low levels, consistent with their similar pure-strategy performance. SelectAL temporarily reduces CoreSet's allocation in rounds 3--5, and keeps ProbCover proportions high. The bandit baselines (d--f) oscillate between strategies without converging. Lines show the mean and shaded bands show $\pm$ 1 standard deviation. For test accuracy, standard deviation is computed across seeds after averaging over splits; for allocations, across splits after averaging over seeds.}
  \label{fig:tcell_frac_evolution}
\end{figure}

\end{document}